\documentclass[12pt]{article}
\usepackage[round,colon,authoryear]{natbib}
\usepackage{bibentry}
\RequirePackage[OT1]{fontenc}
\RequirePackage{amsthm,amsmath,amsfonts,amssymb,{mathtools}}
\RequirePackage[colorlinks,citecolor=blue,urlcolor=blue]{hyperref}
\usepackage{bbm}
\RequirePackage{epsfig, graphicx, color}
\RequirePackage{latexsym}
\RequirePackage{psfrag}

\usepackage{fullpage}
\usepackage{epsf,epsfig,subfigure}
\usepackage{graphics,graphicx}

\usepackage{algorithm,algorithmicx,algpseudocode}

\newtheorem{theorem}{Theorem}
\newtheorem{lemma}{Lemma}

\newtheorem{proposition}{Proposition}
\newtheorem{corollary}{Corollary}

\renewcommand{\hat}{\widehat}

\renewcommand{\hat}{\widehat}

\newcommand{\bfm}[1]{\ensuremath{\mathbf{#1}}}

   \def\bA{\bfm A}  
   \def\bB{\bfm B}  

   \def\bC{\bfm C}  
   \def\bD{\bfm D}  
\def\be{\bfm e}     \def\EE{\mathbb{E}}

   \def\bM{\bfm M}  
   \def\bN{\bfm N}  
     
   \def\bP{\bfm P}  \def\PP{\mathbb{P}}
     
   \def\bR{\bfm R}  \def\RR{\mathbb{R}}
   \def\bS{\bfm S}  
   \def\bT{\bfm T}

   \def\bX{\bfm X}  
   \def\bY{\bfm Y}  
   \def\bZ{\bfm Z}  

\def\calA{{\cal  A}} 
\def\calB{{\cal  B}}

\def\calE{{\cal  E}}

\def\calM{{\cal  M}} 
\def\calN{{\cal  N}} 
 
\def\calP{{\cal  P}}

\def\calT{{\cal  T}}

\newcommand{\bfsym}[1]{\ensuremath{\boldsymbol{#1}}}

            \def\bDelta {\bfsym {\Delta}}

\DeclareMathOperator{\sgn}{sgn}

\DeclareMathOperator{\tr}{tr}

\def\newpage{\vfill\eject}

\newdimen\biblioindent    \biblioindent=30pt

 at 8truept

\def\sgn{\mbox{sgn}}

\newcommand{\beq}{\begin{equation}}
  \newcommand{\eeq}{\end{equation}}
\newcommand{\beqn}{\begin{eqnarray}}
  \newcommand{\eeqn}{\end{eqnarray}}
\newcommand{\beqnn}{\begin{eqnarray*}}
  \newcommand{\eeqnn}{\end{eqnarray*}}

\newcounter{CondCounter}

\usepackage{mathtools}
\DeclarePairedDelimiter\ceil{\lceil}{\rceil}
\DeclarePairedDelimiter\floor{\lfloor}{\rfloor}

\begin{document}
\title{Statistically Optimal and Computationally Efficient Low Rank Tensor Completion from Noisy Entries}
\date{(\today)}

\author{Dong Xia$^\ast$, Ming Yuan$^\ast$ and Cun-Hui Zhang$^\dag$\\
$^{\ast}$Columbia University and $^\dag$Rutgers University}

\footnotetext[1]{Department of Statistics, Columbia University, 1255 Amsterdam Avenue, New York, NY 10027. The research of Dong Xia and Ming Yuan was supported in part by NSF Grants DMS-1265202 and DMS-1721584.}
\footnotetext[2]{Department of Statistics and Biostatistics, Rutgers University, Piscataway, NJ 08854. The research of Cun-Hui Zhang was supported in part by NSF Grants DMS-1129626, DMS-1209014 and DMS-1721495.}

\maketitle

\begin{abstract}
In this article, we develop methods for estimating a low rank tensor from noisy observations on a subset of its entries to achieve both statistical and computational efficiencies. There have been a lot of recent interests in this problem of noisy tensor completion. Much of the attention has been focused on the fundamental computational challenges often associated with problems involving higher order tensors, yet very little is known about their statistical performance. To fill in this void, in this article, we characterize the fundamental statistical limits of noisy tensor completion by establishing minimax optimal rates of convergence for estimating a $k$th order low rank tensor under the general $\ell_p$ ($1\le p\le 2$) norm which suggest significant room for improvement over the existing approaches. Furthermore, we propose a polynomial-time computable estimating procedure based upon power iteration and a second-order spectral initialization that achieves the optimal rates of convergence. Our method is fairly easy to implement and numerical experiments are presented to further demonstrate the practical merits of our estimator.
\end{abstract}

\newpage
\section{Introduction}
Let $\bT\in\mathbb{R}^{d_1\times\cdots\times d_k}$ be a $k$th order tensor, or multilinear array. In the noisy tensor completion problem, we are interested in recovering $\bT$ from observations of a subset of its entries. More specifically, our sample consists of $n$ independent copies $\{(Y_i, \omega_i): 1\le i\le n\}$ of a random pair $(Y, \omega)$ obeying
\begin{equation}
\label{eq:model}
Y=T(\omega)+\xi,
\end{equation}
where $\omega$ is uniformly sampled from $[d_1]\times\ldots\times[d_k]$ where $[d]=\{1,2,\ldots,d\}$, and independent of the measurement error $\xi$ that is assumed to be a centered subgaussian random variable. Of particular interest here is the high dimensional settings where the sample size $n$ may be much smaller than the ambient dimenision $d_1\cdots d_k$. In this case, it may not be possible to estimate an arbitrary $k$th order tensor well but it is possible to do so if we focus on tensors that resides in a manifold of lower dimension in $\RR^{d_1\times\cdots\times d_k}$. A fairly general and practically appropriate example is the class of tensors of low rank. Problems of this type arise naturally in a wide range of applications including imaging and computer vision \citep[e.g.,][]{LiLi10,liu2013tensor,xu2013block}, signal processing \citep[e.g.,][]{lim2010multiarray,SidNion10, kreimer2013tensor,Semerci14}, latent variable modeling \citep[e.g.,][]{CohenCollins12, chaganty2013spectral, anandkumar2014tensor, xie2016topicsketch}, to name a few. Although many statistical methods and algorithms have been proposed for these problems, very little is known about their theoretical properties and to what extent they work and may not work. %

An exception is the special case of matrices, that is $k=2$, for which low rank completion from noisy entries is well-understood. See, e.g., \cite{candes2010matrix, keshavan2010matrix, koltchinskii2011nuclear, rohde2011estimation, klopp2014noisy} and references therein. In particular, as shown by \cite{koltchinskii2011nuclear}, an estimator based on nuclear norm regularization, denoted by $\hat{\bT}^{\rm KLT}$, converges to $\bT$ at the rate of
\begin{equation}\label{eq:matrixcompletion}
\frac{\|\hat{\bT}^{\rm KLT}-\bT\|_{\ell_2}}{(d_1d_2)^{1/2}}=O_p\left((\|\bT\|_{\ell_\infty}\vee \sigma_\xi)\sqrt{\frac{r(d_1\vee d_2)\log(d_1\vee d_2)}{n}}\right),
\end{equation}
where $a\vee b=\max\{a,b\}$, and $\|\cdot\|_{\ell_p}$ ($p\ge 1$) denotes the vectorized $\ell_p$ norm. Hereafter, all limits are taken as the dimension $d_j$s tend to infinity. Note that the dimension of the manifold of rank $r$ matrices is of the order $r(d_1+d_2)$, the aforementioned convergence rate is therefore expected to be optimal, up to the logarithmic factor. Indeed a rigorous argument was given by \cite{koltchinskii2011nuclear} to show that it is optimal, up to the logarithmic factor, in the minimax sense. In contrast to the matrix case, our understanding of higher order tensors ($k\ge 3$) is fairly limited. The main goal of the current work is to fill in this void by establishing minimax lower bounds for estimating $\bT$, and developing computationally efficient methods that attain the optimal statistical performance.

Treatment of higher order tensors poses several fundamental challenges. On the one hand, many of the basic tools and properties for matrices, particularly those pertaining to low rank approximation, are no longer valid for higher order tensors. For example, many of the aforementioned estimating procedures developed for matrices are based on singular value decompositions whose generalization to tensors, however, is rather delicate. As a result, although many of these approaches have been extended to higher order tensors in recent years, their theoretical properties remain largely unclear. And recent studies on a related problem, namely nuclear norm minimization for exact tensor completion without noise, point to many fundamental differences between matrices and higher order tensors despite their superficial similarities. See, e.g., \cite{yuan2015tensor,yuan2016incoherent}. On the other hand, as pointed out by \cite{hillarlim13}, most computational problems related to higher order tensors, including the simple task of evaluating tensor spectral and nuclear norms, are typically NP-hard. This dictates that it is essential to take computational efficiency into account in devising estimating procedures for $\bT$.

Because of these difficulties, results for higher order noisy tensor completion comparable to \eqref{eq:matrixcompletion} are scarce. The strongest result to date is due to \cite{barak2016noisy}. They focused on the case of third order tensors, that is $k=3$, and proved that, under suitable conditions which we shall discuss in details later on, there is a polynomial-time computable estimator, denoted by $\hat{\bT}^{\rm BM}$, such that
\begin{equation}
\label{eq:sosrates}
\frac{1}{d_1d_2d_3}\|\hat{\bT}^{\rm BM}-\bT\|_{\ell_1}=O_p \left(\frac{\|\bT\|_{\ast}(d_{\min}d_{\max}^2)^{1/4}\log^2(d_{\max})}{\sqrt{n}}+{\|\Xi\|_{\ell_1}\over d_1d_2d_3}+{\|\Xi\|_{\ell_\infty}\over \sqrt{n}}\right),
\end{equation}
where $d_{\min}=d_1\wedge d_2\wedge d_3$, $d_{\max}=d_1\vee d_2\vee d_3$, $\|\cdot\|_\ast$ stands for the tensor nuclear norm, and $\Xi$ is a $d_1\times d_2\times d_3$ random tensor whose entries are independent copies of $\xi$. More recently, \cite{montanari2016spectral} considered approximating a general $k$th order $d\times \cdots\times d$ cubic tensor in the absence of noise, that is $\Xi=0$, and proposed a spectral method that yield an estimator $\hat{\bT}^{\rm MS}$ obeying
\begin{equation}
\label{eq:spectralrates}
\|\hat{\bT}^{\rm MS}-\bT\|_{\ell_2}=O_p \left(\frac{\|\bT\|_{\ell_2}r^{1/3}d^{k/6}\log^3(d)}{n^{1/3}}\right),
\end{equation}
under more restrictive assumptions, where $r$ is the rank of $\bT$. However, it remains unknown to what extent these bounds \eqref{eq:sosrates} and \eqref{eq:spectralrates} can be improved, especially if we take computational efficiency into account. The present article addresses this question specifically, and provides an affirmative answer.

In particular, we investigate the minimax optimal estimates for a low rank tensor under the general $\ell_p$ ($1\le p\le 2$) loss. We propose a computational efficient procedure based on low rank projection of an unbiased estimate of $\bT$, and show that, if $\bT$ is well conditioned, then the estimation error of the resulting estimate, denoted by $\hat{\bT}$, satisfies
\begin{equation}
\label{eq:rates}
\left({1\over d_1\cdots d_k}\|\hat{\bT}-\bT\|_{\ell_p}^p\right)^{1/p}=O_p\left((\|\bT\|_{\ell_\infty}\vee \sigma_\xi)\sqrt{\frac{rd_{\max}\log(d_{\max})}{n}}\right),
\end{equation}
provided that
\begin{equation}
\label{eq:sample}
n\ge C\left(r^{(k-1)/2}(d_1\cdots d_k)^{1/2}\log^{k+2}(d_{\max})+r^{k-1}d_{\max}\log^{2(k+2)}(d_{\max})\right),
\end{equation}
for a suitable constant $C>0$. Here $d_{\max}=d_1\vee \cdots\vee d_k$ as before. The above result continues to hold if $r$ represents the maximum multilinear rank of $\bT$. Note that under the typical incoherent assumptions for $\bT$, $\|\bT\|_{\ast}\ge \|\bT\|_{\ell_2}\asymp (d_1 \cdots d_k)^{1/2}\|\bT\|_{\ell_\infty}$ as we shall discuss in further details later. Therefore $\hat{\bT}$ converges to $\bT$ at a much faster rate than both $\hat{\bT}^{\rm BM}$ and $\hat{\bT}^{\rm MS}$. Furthermore, we show that the rates given by \eqref{eq:rates} are indeed minimax optimal, up to the logarithmic factor, over all incoherent tensors of low multilinear rank.

Our estimator also has its practical appeal when compared with earlier proposals. In general, computing the best low rank approximation to a large tensor is rather difficult. See, e.g., \cite{silvalim08} and \cite{hillarlim13}. The root cause of the problem is the highly nonconvex nature of the underlying optimization problem. As a result, there could be exponentially many local optima \citep[see, e.g.,][]{auffinger2013random, auffinger2013complexity}. To address this challenge, we devise a strategy that first narrows down the search area for best low rank approximation using a novel spectral method and then applies power iteration to identify a local optimum within the search area. The high level idea of combining spectral method and power iterations to yield improved estimate is similar in spirit to the classical one-step MLE. Existing polynomial-time computable estimators such as $\hat{\bT}^{\rm BM}$ often involve solving huge semidefinite programs which are known not to scale well for large problems. In contrast, our approach is not only polynomial-time computable but also very easy to implement.

It is worth pointing out that, in order to achieve the minimax optimal rate of convergence given by \eqref{eq:rates}, a sample size requirement \eqref{eq:sample} is imposed. This differs from the matrix case and appears to be inherent to tensor related problems. More specifically, \cite{barak2016noisy} argued that, if there is no polynomial-time algorithm for refuting random 3-SAT of $d$ variables with $d^{3/2-\epsilon}$ clauses for any $\epsilon>0$, then any polynomial-time computable estimator of a $d\times d\times d$ tensor $\bT$ is inconsistent whenever $n=O(d^{3/2-\epsilon})$. This suggests that the sample size requirement of the form \eqref{eq:sample} is likely necessary for any polynomial-time computable estimator because thus far, indeed there is no polynomial-time algorithm for refuting random 3-SAT of $d$ variables with $o(d^{3/2})$ clauses in spite of decades of pursuit.

Our work here is also related to a fast growing literature on exact low rank tensor completion where we observe the entries without noise, that is $\xi=0$ in \eqref{eq:model}, and aim to recover $\bT$ perfectly. See, e.g., \cite{jain2014provable, yuan2015tensor, yuan2016incoherent, xia2017polynomial} and references therein. The two types of problems, albeit connected, have many fundamental differences which manifest prominently in their respective treatment. On the one hand, the noisy completion considered here is technically more involved because of the presence of measurement error $\xi$. In fact, much of our analysis is devoted to carefully control the adverse effect of $\xi$. On the other hand, our interest in the noisy case is in seeking a good estimate or approximation of $\bT$, which is to be contrast with the noiseless case where the goal is for exact recovery and therefore more difficult to achieve. As we shall demonstrate later, this distinction allows for simpler algorithms and sharper analysis in the noisy setting.

The rest of the paper is organized as follows. We first describe the proposed estimation method in Section \ref{sec:meth}. Some useful spectral bounds are given in Section \ref{sec:prelim} which we shall use to establish theoretical properties for our estimator in Section \ref{sec:main}.  Numerical experiments are presented in Section~\ref{sec:simulation} to complement our theoretical results. Proofs of the main results are presented in Section~\ref{sec:proofs}.

\section{Methodology}
\label{sec:meth}
We are interested in estimating a tensor $\bT\in \RR^{d_1\times\cdots\times d_k}$ based on observations $\{(Y_i,\omega_i): 1\le i\le n\}$ that follow
$$
Y_i=T(\omega_i)+\xi_i,\qquad i=1,\ldots,n,
$$
assuming that $\bT$ is of low rank. To this end, we first review some basic notions and facts about tensors.

\subsection{Background and notation}
Recall that the mode-$j$ fibers of a $k$th order tensor $\bA\in \RR^{d_1\times \cdots\times d_k}$ are the $d_j$ dimensional vectors
$$
\{A(i_1,\ldots,i_{j-1},\cdot,i_{j+1},\ldots,i_k): i_1\in [d_1],\ldots, i_k\in [d_k]\},
$$
that is, vectors obtained by fixing all but the $j$th indices of $\bA$. Let $\calM_j(\bA)$ be the mode-$j$ flattening of $\bA$ so that $\calM_j(\bA)$ is a $d_j\times (d_1\cdots d_{j-1}d_{j+1}\cdots d_k)$ matrix whose column vectors are the mode-$j$ fiber of $\bA$. For example, for third order tensor
$\bA\in \RR^{d_1\times d_2\times d_3}$, define the matrix $\calM_1(\bA)\in \RR^{d_1\times (d_2d_3)}$ by the entries
$$
\calM_1(\bA)(i_1,(i_2-1)d_3+i_3)=A(i_1,i_2,i_3),\qquad \forall i_1\in[d_1], i_2\in[d_2], i_3\in[d_3].
$$
Denote by $r_j(\bA)$ the rank of $\calM_j(\bA)$. The tuple $(r_1(\bA),\ldots,r_k(\bA))$ is often referred to the multilinear ranks of $\bA$. Denote by $r_{\max}(\bA):=\max\big\{r_1(\bA),\ldots,r_k(\bA)\big\}$. 

Another common notion of tensor rank is the so-called canonical polyadic (CP) rank. Recall that a rank-one tensor can be expressed as
$$
\bA=u_1\otimes\cdots\otimes u_k,
$$
for some $u_j\in \RR^{d_j}$. Here the $\otimes$ stands for the outer product so that
$$
(u_1\otimes\cdots\otimes u_k)(i_1,\ldots, i_k)=u_{1}(i_1)\cdot\cdots\cdot u_k(i_k).
$$
The CP rank, or sometimes rank for short, of a tensor $\bA$, denoted by $R(\bA)$, is defined as the smallest number of rank-one tensors needed to sum up to $\bA$. It is common in the literature to refer to a tensor as low rank if its CP rank is small. It is clear that $r_{\max}(\bA)\le R(\bA)\le r_1(\bA)\cdots r_k(\bA)/r_{\max}(\bA)$ so that a tensor of low CP rank necessarily has small multilinear ranks. We shall focus on tensors with low multilinear ranks in the rest of the paper because of this connection between the two notions of tensor ranks.

Suppose that we know a priori that $\bT$ is of low rank. A natural starting point for estimating $\bT$ is the least squares estimate:
$$
\min_{\bA\in \Theta(r_1,\ldots, r_k)}\left\{{1\over n}\sum_{i=1}^{n} \left(Y_i-A(\omega_i)\right)^2\right\}=\min_{\bA\in \Theta(r_1,\ldots, r_k)}\left\{{1\over n}\sum_{i=1}^{n} A^2(\omega_i)-{2\over n}\sum_{i=1}^{n} Y_iA(\omega_i)\right\},
$$
where, with slight abuse of notation,
$$
\Theta(r_1,\ldots, r_k):=\left\{\bA\in \RR^{d_1\times\cdots\times d_k}: r_k(\bA)\le r_k\right\}
$$
is the collection of $k$th order tensors whose multilinear ranks are upper bounded by $r_1, \ldots, r_k$ respectively. Note that similarly defined least squares estimator for tensors with small CP rank may not be well defined \citep[see, e.g.,][]{silvalim08}. Noting that $\omega_i$s are uniformly sampled from $[d_1]\times\cdots [d_k]$, we shall replace the first term on the right hand side simply by its population counterpart $\|\bA\|_{\ell_2}^2$, leading to
$$
\min_{\bA\in \Theta(r_1,\ldots, r_k)}\left\{(d_1\cdots d_k)^{-1}\|\bA\|_{\ell_2}^2-{2\over n}\sum_{i=1}^{n} Y_iA(\omega_i)\right\}=\min_{\bA\in \Theta(r_1,\ldots, r_k)}\left\{\left\|{d_1\cdots d_k\over n}\sum_{i=1}^n Y_i \be_{\omega_i}-\bA\right\|_{\ell_2}^2\right\}.
$$
Here $\be_\omega$ is a tensor whose entries are zero except that its $\omega$th entry is one. In other words, we can estimate $\bT$ by the best multilinear ranks-$(r_1,\ldots, r_k)$ approximation of
\begin{equation}
\label{eq:defTinit}
\hat{\bT}^{\rm init}:={d_1\cdots d_k\over n}\sum_{i=1}^n Y_i \be_{\omega_i}.
\end{equation}

Similar approach can also be applied to more general sampling schemes, and was first introduced by \cite{koltchinskii2011nuclear} in the matrix setting. However, there is a major challenge when doing so for higher order tensors: computing low rank approximations to a higher order  ($k\ge 3$) tensor is NP-hard in general \citep[see, e.g.,][]{hillarlim13}. This makes it practically less meaningful to study the properties of the exact projection of $\hat{\bT}^{\rm init}$ onto $\Theta(r_1,\ldots, r_k)$. To overcome this hurdle, we adapted the power iteration as a way to compute an ``approximate'' projection. We shall show in later sections that, even though it may not produce the true projection, running the algorithm for a sufficiently large but finite number of iterations is guaranteed to yield an estimate that attains the minimax optimal rate of convergence.

\subsection{Estimation via power iterations}
\label{subsec:alg}

Recall that we are interested in solving
\begin{equation}
\label{eq:obj}
\min_{\bA\in \Theta(r_1,\ldots, r_k)}\left\{\left\|\hat{\bT}^{\rm init}-\bA\right\|_{\ell_2}^2\right\}
\end{equation}
The objective function is smooth so that many smooth optimization algorithms might be employed. In particular, we shall consider using power iterations, one of the most common methods for low rank approximation. See, e.g., \cite{kolda2009tensor}.

For a tensor $\bA\in \RR^{d_1\times\cdots\times d_k}$, denote by $U_j$ the left singular vectors of $\calM_j(\bA)$. Then we can find a tensor $\bC\in \RR^{r_1(\bA)\times\cdots\times r_k(\bA)}$ such that
$$
\bA=\bC\times_1 U_1^\top\times_2\cdots\times_k U_k^\top.
$$
Here the marginal multiplication $\times_j$ between a tensor $\bA$ and a matrix $B$ of conformable dimension results in a tensor whose entries are defined as
$$
(\bA\times_j B)(i_1,\ldots, i_{j-1},i_j,i_{j+1},\ldots,i_k)=\sum_{i'_j} A(i_1,\ldots, i_{j-1},i'_j,i_{j+1},\ldots,i_k)B(i'_j,i_j).
$$
In particular, it can be derived that, if $\bA$ is the solution to \eqref{eq:obj}, then
$$
\bC=\hat{\bT}^{\rm init}\times_{j=1}^k U_j,
$$
and $U_j$ is a $d_j\times r_j$ matrix whose columns are the leading singular vectors of
$$
\calM_j\big(\hat{\bT}^{\rm init}\times_{j'\neq j} U_{j'}\big).
$$
This naturally leads to Algorithm \ref{alg:power} which updates $\bC$ and $U_j$s in an iterative fashion.

\begin{algorithm}
 \caption{{\bf Power Iterations}}\label{alg:power}
  \begin{algorithmic}[2]
  \State {\bf Input:} $\hat{\bT}^{\rm init}$, $U_j^{(0)}$ for $j=1,2,\ldots,k$, and parameter ${\rm iter}_{\max}$.
  \State {\bf Output:} $\hat{\bT}$.
  
  \State Set counter ${\rm iter}=0$.  
   \While{${\rm iter}<{\rm iter}_{\max}$}
   \State Set ${\rm iter}={\rm iter}+1$ and $j=0$.
   \While{$j<k$}
   \State Set $j=j+1$.
   \State Set $U^{({\rm iter})}_j$ to be the first $r_j$ left singular vectors of
$$
\calM_j\left(\hat{\bT}^{\rm init}\times_{j'<j} U^{({\rm iter})}_{j'}\times_{j'>j}U_{j'}^{({\rm iter}-1)}\right).
$$
   \EndWhile
   \EndWhile
  \State Return $\hat{\bT}=\hat{\bT}^{\rm init}\times_{j=1}^k U_j^{({\rm iter})}\left(U_j^{({\rm iter})}\right)^\top$.
  \end{algorithmic}
\end{algorithm}

The power iteration as described above is guaranteed to converge for any given initial value $U_j^{(0)}$s. But it is only guaranteed to converge to a local optimum of \eqref{eq:obj}. See, e.g., \cite{uschmajew2012local} and references therein for further discussion about the convergence of power method.

\subsection{Spectral initialization}
\label{sec:spectral}
Because of the highly nonconvex nature of the space of low rank tensors, the local convergence of Algorithm \ref{alg:power} may not ensure that it yields a good estimate. For example, as shown by \cite{auffinger2013random}, there could be exponentially many (in $d$s) local optima and vast majority of these local optima are far from the best low rank approximation. See also \cite{auffinger2013complexity}. An observation key to our development is that if we start from an initial value not too far from the global optimum, then a local optimum reached by these locally convergent algorithms may be as good an estimator as the global optimum. In fact, as we shall show later, if we start from an appropriate initial value, then even running Algorithm \ref{alg:power} for a finite number of iteration could yield a high quality estimate of $\bT$.

It turns out, however, that the construction of an initial value for $U_j$s that are both close to the truth, i.e., the leading left singular vectors of $\calM_j(\bT)$, and polynomial-time computable is a fairly challenging task. An obvious choice is to start with higher order singular value decomposition \citep[HOSVD; see, e.g.,][]{de2000multilinear}, and initialize $U_j$ as the left singular vectors of $\calM_j(\hat{\bT}^{\rm init})$. It is clear that how close such an initialization is to the truth is determined by the difference $\calM_j(\hat{\bT}^{\rm init})-\calM_j(\bT)$. This approach, however, neglects the fact that we are only interested in the left singular vectors of a potentially very ``fat'' ($d_j\ll (d_1\cdots d_k)/d_j$) matrix. As a result, it can be shown that an unnecessarily large amount of samples are needed to ensure that such an initialization is sufficiently ``close'' to the truth.

To overcome this difficulty, we adopt a second order spectral method developed by \cite{xia2017polynomial}. Note that the column vectors of left singular vectors of $\calM_j(\bT)$ are also the leading eigenvectors of
$$
\bN_j:=\calM_j(\bT)\calM_j(\bT)^\top.
$$
Therefore, we could consider estimating the eigenvectors of $\bN_j$ instead. Specifically, we first estimate $\bN_j$ by the following $U$-statistic:
\begin{equation}\label{eq:U-statistics}
\hat{\bN}_j:=\frac{(d_1\cdots d_k)^2}{n(n-1)}\sum_{1\leq i\neq i'\leq n}Y_iY_{i'}\calM_j(\be_{\omega_i})\calM_j(\be_{\omega_{i'}})^\top.
\end{equation}
We then initialize $U_j$ as the collection of eigenvectors of $\hat{\bN}_j$ with eigenvalues greater than a threshold $\lambda$ to be specified later.
We shall show later that $\hat{\bN}_j$ concentrates around $\bN_j$ than $\calM_j(\hat{\bT}^{\rm init})$ around $\calM_j(\bT)$ and therefore allows for better initialization of $U_j$s. In summary, our estimating procedure can be described by Algorithm \ref{alg:hosvd}.

\begin{algorithm}[htbp]
 \caption{{\bf Compute Estimate of $\bT$ from $\{(Y_i,\omega_i): 1\le i\le n\}$}}\label{alg:hosvd}
  \begin{algorithmic}[2]
  \State {\bf Input:} Observations $\{(Y_i,\omega_i): 1\le i\le n\}$, threshold $\lambda$, and parameter ${\rm iter}_{\max}$.
  \State {\bf Output:} $\hat{\bT}$.

  \State Compute $\hat{\bT}^{\rm init}$ as given by \eqref{eq:defTinit}.  
  \State Initialize $U_j$s:
   \For{$j=1,\ldots k$}
   \State Compute $\hat{\bN}_j$ as given by \eqref{eq:U-statistics}.
   \State Compute the eigenvectors, denoted by $U_j^{(0)}$, of $\hat{\bN}_j$ with eigenvalue greater than $\lambda^2$.
   \EndFor
  \State Run Algorithm \ref{alg:power} with inputs $\hat{\bT}^{\rm init}$, $U_j^{(0)}$s and ${\rm iter}_{\max}$ to get $\hat{\bT}$.
  \State Return $\hat{\bT}$.
  \end{algorithmic}
\end{algorithm}

We now turn our attention to the theoretical properties of the proposed estimating procedure, and show that with appropriately chosen threshold $\lambda$, we can ensure that the estimate produced by Algorithm \ref{alg:hosvd} is minimax optimal under suitable conditions. Before proceeding, we need a couple of probabilistic bounds regarding the quality of $\hat{\bT}^{\rm init}$ and $\hat{\bN}_j$, respectively.

\section{Preliminary Bounds}
\label{sec:prelim}
It is clear that the success of our estimating procedure hinges upon how close $\hat{\bT}^{\rm init}$ is to $\bT$, and $\hat{\bN}_j$ to $\bN_j$. We shall begin by establishing spectral bounds on them, which might also be of independent interest.

\subsection{Bounding the spectral norm of $\hat{\bT}^{\rm init}-\bT$}
\label{sec:general}

We first consider bounding the spectral norm of $\hat{\bT}^{\rm init}-\bT$. Write, for two tensors $\bA, \bB\in \RR^{d_1\times \ldots\times d_k}$,
$$
\langle \bA, \bB\rangle =\sum_{\omega\in [d_1]\times\cdots\times [d_k]} A(\omega)B(\omega)
$$
as their inner product. The spectral norm of a tensor $\bA$ is given by
$$
\|\bA\|=\max_{\substack{u_j\in \RR^{d_j}: \|u_1\|_{\ell_2},\ldots,\|u_k\|_{\ell_2}\le1}}\langle\bA, u_1\otimes\cdots\otimes u_k\rangle.
$$

The following theorem provides a probabilistic bound on the spectral norm of the difference $\hat{\bT}^{\rm init}-\bT$. \begin{theorem}
\label{th:init}
Assume that $\xi$ is subgassian in that there exits a $\sigma_\xi>0$ such that for all $s\in\mathbb{R}$,
$$\mathbb{E}\left(\exp\left\{s\xi\right\}\right)\leq \exp\left(s^2\sigma_{\xi}^2/2\right).$$
There exists a numerical constant $C>0$ such that, for any $\alpha\geq 1$, 
$$
\|\hat{\bT}^{\rm init}-\bT\|\le Ck^{k+3}\alpha\big(\|\bT\|_{\ell_\infty}\vee \sigma_{\xi}\big)\max\bigg\{\sqrt{\frac{kd_{\max}d_1\ldots d_k}{n}},\quad \frac{kd_1\ldots d_k}{n} \bigg\}\log^{k+2}d_{\max},
$$
with probability at least $1-d_{\max}^{-\alpha}$.
\end{theorem}

In the matrix case, that is $k=2$, the bound given by Theorem \ref{th:init} is essentially the same as those from \cite{koltchinskii2011nuclear}. More importantly, Theorem \ref{th:init} also highlights a key difference between matrices ($k=2$) and higher order tensors ($k\ge 3$). To fix ideas, consider, for example, the case when $\|\bT\|_{\ell_\infty}, \sigma_{\xi}\asymp (d_1\cdots d_k)^{-1/2}$. Theorem \ref{th:init} implies that
\begin{equation}
\label{eq:initbd}
\|\hat{\bT}^{\rm init}-\bT\|=O_p\left(\max\left\{\sqrt{\frac{kd_{\max}}{n}},\frac{ k(d_1\cdots d_k)^{1/2}}{n}\right\}{\rm polylog}(d_{\max})\right),
\end{equation}
where ${\rm polylog}(\cdot)$ is a certain polynomial of the logarithmic function.  The first term on the right hand side of \eqref{eq:initbd} can be shown to be essentially optimal. In the matrix case, this is indeed the dominating term, the very reason why the best low rank approximation of $\hat{\bT}^{\rm init}$ is a good estimate of $\bT$. For higher order tensors, however, this is no longer the case, and two different rates of convergence emerge depending on the sample size. The first term is the leading term only if
$$
n\gg d_{\max}^{-1}(d_1\cdots d_k),
$$
Yet, for smaller sample sizes, the second term dominates so that
$$
\|\hat{\bT}^{\rm init}-\bT\|=O_p\left(\frac{(d_1\cdots d_k)^{1/2}{\rm polylog}(d_{\max})}{n}\right).
$$
In particular, $\hat{\bT}^{\rm init}$ is consistent in terms of spectral norm if
$$
n\gg \max\{d_{\max},(d_1\cdots d_k)^{1/2}\}\cdot{\rm polylog}(d_{\max}).
$$
In a way, this is why we need the sample size requirement such as \eqref{eq:sample}. It is in place to ensure that $\hat{\bT}^{\rm init}$ is an consistent estimate of $\bT$ in the sense of tensor spectral norm.

\subsection{Bounding the spectral norm of $\hat{\bN}_j-\bN_j$}
Now consider bounding $\|\hat{\bN}_j-\bN_j\|$. To this end, write
$$
\Lambda_{\min}(\bA)=\min\big\{\sigma_{\min}\big(\calM_1(\bA)\big),\ldots,\sigma_{\min}\big(\calM_k(\bA)\big)\big\}
$$
and
$$
\Lambda_{\max}(\bA)=\max\big\{\sigma_{\max}\big(\calM_1(\bA)\big),\ldots,\sigma_{\max}\big(\calM_k(\bA)\big)\big\}
$$
where $\sigma_{\min}(\cdot)$ and $\sigma_{\max}(\cdot)$ denote the smallest and largest nonzero singular values of a matrix. Then the condition number of $\bA$ is defined as
$$
\kappa(\bA):=\frac{\Lambda_{\max}(\bA)}{\Lambda_{\min}(\bA)}.
$$
Recall that
$$
\bN_j:=\calM_j(\bT)\calM_j(\bT)^\top.
$$
and
$$
\hat{\bN}_j:=\frac{(d_1\cdots d_k)^2}{n(n-1)}\sum_{1\leq i\neq i'\leq n}Y_iY_{i'}\calM_j(\be_{\omega_i})\calM_j(\be_{\omega_{i'}})^\top.
$$
\cite{xia2017polynomial} proved that when $k=3$ and $\sigma_\xi=0$, $\hat{\bN}_j$ is a good estimate of $\bN_j$. Our next result shows that this is also true for more general situations.

\begin{theorem}\label{thm:N1bound}
There exist absolute constants $C_1, C_2>0$ such that, for any $\alpha\geq 1$, if
$$
n\geq C_1\alpha\left(\sqrt{d_1\cdots d_k}\log d_{\max}+d_{\max}\log^2 d_{\max}\right),
$$
then, with probability at least $1-d_{\max}^{-\alpha}$,
\begin{eqnarray*}
\|\hat{\bN}_j-\bN_j\|&\leq& C_2\bigg(\left(\sigma_{\xi}+\|\bT\|_{\ell_\infty}\right)\|\calM_j(\bT)\|\sqrt{\frac{\alpha kd_jd_1\cdots d_k\log d_{\max}}{n}}+\\
&&+ \alpha^3\left(\|\bT\|_{\ell_\infty}^2+\sigma_{\xi}^2\log^2 d_{\max}\right)\frac{(kd_1\cdots d_k)^{3/2}\log^3 d_{\max}}{n}\left(1+\sqrt{{d_j^2}/(d_1\cdots d_k)}\right)\bigg).
\end{eqnarray*}
\end{theorem}

Let $U_j$ and $\hat{U}_j$ be the top-$r_j$ left singular vectors of $\bN_j$ and $\hat{\bN}_j$ respectively.  Applying Wedin's $\sin\Theta$ theorem \citep{wedin1972perturbation}, Theorem \ref{thm:N1bound} immediately implies that
\begin{eqnarray*}
\|\hat{U}_j\hat{U}_j^{\top}-U_jU_j^{\top}\|&\leq& C_2\left(\sigma_{\xi}+\|\bT\|_{\ell_\infty}\right)\frac{\|\calM_j(\bT)\|}{\sigma_{\min}^2(\calM_j(\bT))}\sqrt{\frac{\alpha kd_jd_1\cdots d_k\log d_{\max}}{n}}\\
&&+ C_2\alpha^3\frac{\left(\|\bT\|_{\ell_\infty}^2+\sigma_{\xi}^2\log^2 d_{\max}\right)}{\sigma_{\min}^2(\calM_j(\bT))}\frac{(kd_1\cdots d_k)^{3/2}\log^3 d_{\max}}{n}\left(1+\sqrt{\frac{d_j^2}{d_1\cdots d_k}}\right).
\end{eqnarray*}
In other words, $\hat{U}_j$s are consistent estimates of $U_j$s if
\begin{align*}
n\gtrsim \Lambda_{\min}^{-2}(\bT)\max\Big\{&\kappa(\bT)^2(\|\bT\|_{\ell_\infty}\vee \sigma_{\xi})^2d_{\max}d_1\ldots d_k\log d_{\max},\\
&(\|\bT\|_{\ell_\infty}\vee \sigma_{\xi}\log d_{\max})^2(d_1\ldots d_k)^{3/2}\log^3 (d_{\max})\left(1+\sqrt{\frac{d_j^2}{d_1\cdots d_k}}\right)\Big\}.
\end{align*}
In particular, to fix ideas, if we look at the case when $d_1=\cdots=d_k=:d$ and $\bT$ is well behaved in that $\kappa(\bT)$ and $\Lambda_{\min}(\bT)^{-1}$ are bounded from above, then this bound can be simplified as
$$
n\gtrsim (\|\bT\|_{\ell_\infty}\vee \sigma_{\xi})^2d^{3k/2}\cdot{\rm polylog}(d).
$$
This is to be contrasted with the na\"ive HOSVD for which we have
\begin{proposition}
\label{pr:hosvd}
Let $U_j$ and $\hat{U}_j^{\rm HOSVD}$ be the top $r_j$ singular vectors of $\calM_j(\bT)$ and $\calM_j(\hat{\bT}^{\rm init})$ respectively. Then there exists a universal constant $C>0$ such that, for any $\alpha\geq 1$, the following bound holds with probability at least $1-d_{\max}^{-\alpha}$,
\begin{eqnarray*}
\|\hat{U}_j^{\rm HOSVD}(\hat{U}_j^{\rm HOSVD})^\top-U_jU_j^\top\|_{\ell_2}\\
\le C\frac{(\|\bT\|_{\ell_\infty}\vee \sigma_{\xi})}{\sigma_{\min}(\calM_j(\bT))}\times \max\left\{\sqrt{\Big(d_j\vee \frac{d_1\ldots d_k}{d_j}\Big)\frac{\alpha kd_1\ldots d_k\log(d_{\max})}{n}}
, \frac{\alpha kd_1\ldots d_k\log(d_{\max})}{n}\right\}
\end{eqnarray*}
\end{proposition}

By Proposition \ref{pr:hosvd}, in the case of well conditioned cubic tensors, to ensure that $\hat{U}_j^{\rm HOSVD}$s are consistent, we need a sample size
$$
n\gtrsim \max\left\{(\|\bT\|_{\ell_\infty}\vee \sigma_{\xi})d^k\log d,(\|\bT\|_{\ell_\infty}\vee \sigma_{\xi})^2d^{2k-1}\log d\right\},
$$
which is much more stringent than that for $\hat{U}_j$.

\section{Performance Bounds for $\hat{\bT}$}
\label{sec:main}
We are now in position to study the performance of our estimate $\hat{\bT}$, as the output from Algorithm \ref{alg:hosvd}. Our risk bound can be characterized by the incoherence of $\bT$ which we shall first describe. Coherence of a tensor can be defined through the singular space of its flattening. Let $U$ be a $d\times r$ matrix with orthonormal columns. Its coherence is given by as
$$
\mu(U)=\frac{d}{r}\max_{1\leq i\leq d}\left\|U_{i\cdot}\right\|^2_{\ell_2},
$$
where $U_{i\cdot}$ is the $i$th row vector of $U$. Now for a tensor $\bA\in \RR^{d_1\times \cdots\times d_k}$ such that $\calM_j(\bA)=U_j\Sigma_jV_j^\top$ is its thinned singular value decomposition, we can define its coherence by
$$
\mu(\bA)=\max\left\{\mu(U_1), \ldots,\mu(U_k)\right\}.
$$
Coherence of a tensor can also be measured by its spikiness:
$$
\beta(\bA):=(d_1\ldots d_k)^{1/2}\frac{\|\bA\|_{\ell_\infty}}{\|\bA\|_{\ell_2}}.
$$
The spikiness of a tensor is closely related to its coherence. 
\begin{proposition}
\label{pr:coherence}
For any $\bA\in \RR^{d_1\times\cdots\times d_k}$,
$$
\beta(\bA)\le r_1^{1/2}(\bA)\cdots r_k^{1/2}(\bA)\mu^{k/2}(\bA).
$$
Conversely,
$$
\mu(\bA)\le \beta^2(\bA)\kappa^2(\bA).
$$
\end{proposition}

\subsection{General risk bound}
We first provide a general risk bound for $\hat{\bT}$ when the sample size is sufficiently large.

\begin{theorem}\label{th:general}
Assume that $\xi$ is subgassian in that there exits a $\sigma_\xi>0$ such that for all $s\in\mathbb{R}$,
$$\mathbb{E}\left(\exp\left\{s\xi\right\}\right)\leq \exp\left(s^2\sigma_{\xi}^2/2\right).$$
There exist constants $C_1, C_2, C_3, C_4>0$ depending on $k$ only such that for any fixed $\alpha\ge 1$ and $\gamma\ge C_1$, if
$$
n\geq C_2\alpha \left[\kappa(\bT)\beta(\bT)\right]^{2(k-1)}d_{\max}\log d_{\max},
$$
then with probability at least $1-d_{\max}^{-\alpha}$,
\begin{align*}
\frac{1}{(d_1\ldots d_k)^{1/p}}\|\hat\bT-\bT\|_{\ell_p}\leq C_3\gamma^2\alpha^{3/2}\kappa(\bT)&\big(\|\bT\|_{\ell_\infty}\vee \sigma_{\xi}\big)\log^{k+2}d_{\max}\times\\
&\times\bigg(\kappa(\bT)r_{\max}(\bT)^{(k-1)/2}\sqrt{\frac{d_{\max}}{n}}+r_{\max}(\bT)^{1/2}\frac{(d_1\ldots d_k)^{1/4}}{n^{1/2}}+\\
&\hskip 30pt+r_{\max}(\bT)^{(k-1)/2}\frac{(d_1\ldots d_k)^{1/2}}{n}\bigg),
\end{align*}
for all $1\leq p\leq 2$ where $\hat{\bT}$ is the output from Algorithm \ref{alg:hosvd} with ${\rm iter}_{\max}>C_4 \log d_{\max}$, and
\begin{align*}
\lambda=& \gamma\alpha^{3/2}\big(\|\bT\|_{\ell_\infty}\vee \sigma_{\xi}\big)\log^{k+2}d_{\max}\times\\
&\times\bigg(\kappa(\bT)r_{\max}(\bT)^{(k-2)/2}\sqrt{\frac{d_{\max}d_1\ldots d_k}{n}}+\frac{(d_1\ldots d_k)^{3/4}}{n^{1/2}}+r_{\max}(\bT)^{(k-2)/2}\frac{d_1\ldots d_k}{n}\bigg).
\end{align*}
\end{theorem}

We emphasize that Theorem \ref{th:general} applies to any estimate produced by power iteration after an $O(\log d_{\max})$ number of iterations. In other words, it applies beyond the best low rank approximation to $\hat{\bT}^{\rm init}$. We can further simplify the risk bound in Theorem \ref{th:general} when the rank $r_{\max}(\bT)$ is not too big. More specifically,
\begin{corollary}
\label{cor:suboptimal}
Under the assumptions of Theorem \ref{th:general}, if, in addition,
$$n\geq r_{\max}(\bT)^{k-2}(d_1\ldots d_k)^{1/2}\qquad {\rm and}\qquad \kappa(\bT)^2r_{\max}(\bT)^{k-2}d_{\max}\leq (d_1\ldots d_k)^{1/2},$$
then with probability at least $1-d_{\max}^{-\alpha}$,
$$
\frac{1}{(d_1\ldots d_k)^{1/p}}\|\hat\bT-\bT\|_{\ell_p}^p\\
\leq C\gamma^2\alpha^{3/2}\kappa(\bT)\big(\|\bT\|_{\ell_\infty}\vee \sigma_{\xi}\big)r_{\max}(\bT)^{1/2}\frac{(d_1\ldots d_k)^{1/4}}{n^{1/2}}\log^{k+2}d_{\max}
$$
for all $1\leq p\leq 2$, and some constant $C>0$ depending on $k$ only.
\end{corollary}

To gain further insights into the risk bound in Theorem \ref{th:general}, it is instructive to consider the case of cubic tensors, that is $d_1=\ldots=d_k=:d$. By Theorem \ref{th:general}, we have
\begin{align*}
d^{-k/p}\|\hat\bT-\bT\|_{\ell_p}\le C\kappa(\bT)&(\|\bT\|_{\ell_\infty}\vee \sigma_\xi)\log^{k+2}d\times\\
&\times\bigg(\kappa(\bT)r_{\max}(\bT)^{(k-1)/2}\sqrt{\frac{d}{n}}+r_{\max}(\bT)^{1/2}\frac{d^{k/4}}{n^{1/2}}+r_{\max}(\bT)^{(k-1)/2}\frac{d^{k/2}}{n}\bigg).
\end{align*}
This rate of convergence improves those obtained earlier by \cite{barak2016noisy} and \cite{montanari2016spectral} even though their results are obtained under more restrictive conditions. Indeed, we shall now show that, if the tensor $\bT$ is well conditioned, much sharper performance bounds can be established for our estimate.

\subsection{Minimax optimality}
The following result shows that when the sample size is sufficiently large, power iterations starting with a good initial value indeed produces an estimate with the optimal rate of convergence, within a finite number of iterations.

\begin{theorem}
\label{th:power}
Let $\xi$ be subgaussian in that there exits a $\sigma_\xi>0$ such that for all $s\in\mathbb{R}$,
$$\mathbb{E}\left(\exp\left\{s\xi\right\}\right)\leq \exp\left(s^2\sigma_{\xi}^2/2\right).$$ 
There are constants $C_1,C_2,C_3>0$ depending on $k$ only such that the following holds. Let $\breve{\bT}$ be the output from Algorithm \ref{alg:power} with the number of iterations
$$
{\rm iter}_{\max}>C_1 \log d_{\max},
$$
and initial value such that
\begin{equation}
\label{eq:initcond}
\max_{1\le j\le k} \|U_j^{(0)}(U_j^{(0)})^\top -U_jU_j^\top\|\le {1\over 2}.
\end{equation}
For any fixed $\alpha>1$, if
\begin{eqnarray}
\nonumber
n\ge C_2\max\bigg\{\alpha^2r_{\max}(\bT)^{k-2}\Lambda_{\min}^{-2}(\bT)\left(\|\bT\|_{\ell_\infty}\vee \sigma_{\xi}\right)^2d_{\max}(d_1\cdots d_k)\log^{2(k+2)}d_{\max},\nonumber\\
\nonumber
\alpha\big(\beta(\bT)\kappa(\bT)\big)^{2(k-1)}d_{\max}\log (d_{\max}),\\
 \alpha r_{\max}(\bT)^{(k-2)/2}\Lambda_{\min}^{-1}(\bT)\left(\|\bT\|_{\ell_\infty}\vee \sigma_{\xi}\right)d_1\cdots d_k\log^{k+2}d_{\max},\\
 \alpha \kappa(\bT)^2\Lambda_{\min}^{-2}(\bT)\big(\|\bT\|_{\ell_\infty}\vee \sigma_{\xi}\big)^2\big(d_{\max}\vee r_{\max}(\bT)^{k-1}\big)d_1\ldots d_k\log d_{\max}\bigg\},\label{eq:powernreq}
\end{eqnarray}
then, with probability at least $1-d_{\max}^{-\alpha}$,
$$
\frac{1}{(d_1\cdots d_k)^{1/p}}\|\breve{\bT}-\bT\|_{\ell_p}\leq C_3\left(\|\bT\|_{\ell_\infty}\vee \sigma_{\xi}\right)\sqrt{\frac{\alpha( r_{\max}(\bT)d_{\max}\vee r_{\max}(\bT)^k)\log (d_{\max})}{n}},
$$
for all $1\leq p\leq 2$.
\end{theorem}

As an immediate consequence of Theorems \ref{thm:N1bound} and \ref{th:power}, we have
\begin{corollary}
\label{co:main}
Suppose that $\xi$ is subgaussian as in Theorem \ref{th:power}. There exist constants $C_1, C_2, C_3, C_4>0$ depending on $k$ only such that the following holds for any $\alpha>1$ and $\gamma\ge C_1$. Assume that
$$
{\rm iter}_{\max}>C_2\log d_{\max},
$$
and
\begin{align*}
\lambda=& \gamma\alpha^{3/2}\big(\|\bT\|_{\ell_\infty}\vee \sigma_{\xi}\big)\log^{k+2}d_{\max}\times\\
&\times\bigg(\kappa(\bT)r_{\max}(\bT)^{(k-2)/2}\sqrt{\frac{d_{\max}d_1\ldots d_k}{n}}+\frac{(d_1\ldots d_k)^{3/4}}{n^{1/2}}+r_{\max}(\bT)^{(k-2)/2}\frac{d_1\ldots d_k}{n}\bigg).
\end{align*}
If
\begin{eqnarray*}
n\geq C_3\gamma^2\max\bigg\{\alpha^3\big(\kappa^2(\bT)\vee r_{\max}(\bT)^{k-2}\big)\Lambda_{\min}^{-2}(\bT)\big(\|\bT\|_{\ell_\infty}\vee \sigma_{\xi}\big)^2d_{\max}d_1\ldots d_k\log^{2(k+2)}d_{\max},\\
\alpha^3\Lambda_{\min}^{-2}(\bT)(\|\bT\|_{\ell_\infty}\vee \sigma_{\xi})^2(d_1\ldots d_k)^{3/2}\log^{(k+2)}d_{\max},\\
\alpha^{3/2} r_{\max}(\bT)^{(k-2)/2}\Lambda_{\min}^{-1}(\bT)(\|\bT\|_{\ell_\infty}\vee\sigma_{\xi})d_1\ldots d_k\log^{(k+2)}d_{\max},\\
\alpha\big(\beta(\bT)\kappa(\bT)\big)^{2(k-1)}d_{\max}\log d_{\max}\bigg\}\label{eq:maincor_n}
\end{eqnarray*}
then, with probability at least $1-d_{\max}^{-\alpha}$,
$$
\frac{\|\hat{\bT}-\bT\|_{\ell_p}}{(d_1\ldots d_k)^{1/p}}\leq C_4\big(\sigma_{\xi}\vee \|\bT\|_{\ell_\infty}\big)\sqrt{\frac{\alpha(r_{\max}(\bT)d_{\max}\vee r_{\max}(\bT)^k)\log d_{\max}}{n}}
$$
for all $1\leq p\leq 2$.
\end{corollary}

It is instructive to consider the case when $d_1=\cdots=d_k=d$, and $\big(\|\bT\|_{\ell_\infty}\vee \sigma_{\xi}\big)=O\big(\|\bT\|_{\ell_2}(d_1\ldots d_k)^{-1/2}\big)$, then Corollary \ref{co:main} implies that
$$
d^{-k/2}\|\hat{\bT}-\bT\|_{\ell_2}= O\left(\left(\|\bT\|_{\ell_\infty}\vee \sigma_{\xi}\right)\sqrt{\frac{(r_{\max}(\bT)d\vee r_{\max}(\bT)^{k})\log (d)}{n}}\right),
$$
given that
$$
n\gg r_{\max}(\bT)^{(k-1)/2}(d_1\ldots d_k)^{1/2}{\rm polylog}(d).
$$
In particular, when $k=2$, this matches the optimal bounds for noisy matrix completion. See, e.g., \cite{keshavan2010matrix, koltchinskii2011nuclear} and references therein. Indeed, as the next theorem shows that the rate of convergence achieved by $\hat{\bT}$ is indeed minimax optimal up to the logarithmic factor. Let $\PP_{\bT}$ denote the joint distribution of $\{(Y_i,\omega_i): i=1,\ldots,n\}$ with
$$
Y_i=T(\omega_i)+\xi_i,\qquad \xi_i\sim\calN(0,\sigma_{\xi}^2).
$$
Denote by
$$
\Theta(r_0,\beta_0):=\left\{\bA\in \RR^{d_1\times\cdots\times d_k}: r_{\max}(\bA)\le r_0; \beta(\bA)\le \beta_0\right\}.
$$

\begin{theorem}
\label{thm:minimax}
Let $\beta_0\geq 2$. Then, there exist absolute constants $C_1, C_2>0$ such that for any $M\ge 0$,
$$
\inf_{\tilde{\bT}}\sup_{\bT\in\Theta(r_0,\beta_0): \|\bT\|_{\ell_\infty}\le M}\mathbb{P}_{\bT}\left\{(d_1\cdots d_k)^{-1/p}\|\tilde{\bT}-\bT\|_{\ell_p}\geq C_1\left(M\wedge \sigma_{\xi}\right)\sqrt{\frac{r_0d_{\max}\vee r_0^k}{n}}\right\}\geq C_2,
$$
for all $1\leq p\leq 2$, where the infimum $\tilde{\bT}$ is taken over all the estimators based on $\{(Y_i, \omega_i): 1\le i\le n\}$.
\end{theorem}

\subsection{Random tensor model}
To better appreciate the above risk bounds, we now consider a more specific random tensor model previously studied by \cite{montanari2016spectral}. Let $\bT$ be a symmetric tensor with rank $r$ such that
$$
\bT=\sum_{i=1}^r u_i\otimes\ldots\otimes u_i,
$$
where $u_i$s are independent and identically distributed subgaussian random vector in $\RR^d$ with mean $0$ and $\EE( u_i\otimes u_i)=I_d$. It is not hard to see that
$$\big\|\calM_j(\bT)\big\|\asymp_p d^{k/2}.$$ 
See, e.g, \cite{montanari2016spectral}. Here $\asymp_p$ means $\asymp$ with high probability. Meanwhile, it is clear that
\begin{eqnarray*}
\|\bT\|_{\ell_2}^2=\sum_{i=1}^r\big(\|u_i\|_{\ell_2}^2\big)^k+2\sum_{1\leq i<i'\leq r} \langle u_i, u_{i'}\rangle^k\asymp_p rd^{k},
\end{eqnarray*}
so that
$$\sigma_{\min}\big(\calM_j(\bT\big)\asymp_p d^{k/2}.$$
Therefore,
$$
{\Lambda}_{\min}\asymp_p d^{k/2}\qquad {\rm and}\qquad \kappa(\bT)\asymp_p 1.
$$
Moreover, we have $\|\bT\|_{\ell_\infty}\asymp_p r^{1/2}\log^{k/2} d$. If $\sigma_\xi=O(1)$, then Corollary \ref{co:main} implies that, by taking
$$
\lambda=\gamma\left(r^{(k-1)/2}\sqrt{d^{k+1}\over n}+r^{1/2}{d^{3k/4}\over n}+r^{(k-1)/2}{d^k\over n}\right){\rm polylog}(d),
$$
we get
$$
d^{-k/p}\|\hat\bT-\bT\|_{\ell_p}=O_p\left(\sqrt{\frac{dr\log^2 d}{n}}\right)
$$
if
$$
n\geq C\gamma^2\left(dr^{k-1}+r^{(k-1)/2}d^{k/2}\right){\rm polylog}(d),
$$
for some absolute constant $C>0$.
%

\section{Numerical Experiments}\label{sec:simulation}
To complement our theoretical development, we present in this section results from several sets of numerical experiments. We begin with simulated third order tensors where the underlying tensor $\bT$ is generated from the following random tensor model
$$
\bT=\sum_{k=1}^r \lambda (u_k\otimes v_k\otimes w_k)\in\mathbb{R}^{d\times d\times d}
$$
with $\lambda=d^{3/2}$ and $U=[u_1;\ldots;u_r]\in\mathbb{R}^{d\times r}$ (also $V, W$) being randomly generated orthonormal columns from the eigenspace of a standard Gaussian random matrix. It is well known that $\|\lambda (u_k\otimes v_k\otimes w_k)\|_{\ell_\infty}=O(\log^{3/2}d)$ with high probability under such construction. In addition to our proposed estimator, we shall also consider the following estimator:
$$
\hat{\bT}^{(0)}=\hat{\bT}^{\rm init}\times_1 \bP_{U^{(0)}}\times_2 \bP_{V^{(0)}}\times_3 \bP_{W^{(0)}}.
$$
where $U^{(0)}$ ($V^{(0)}, W^{(0)}$ resp.) denotes the spectral initialization from U-statistics in (\ref{eq:U-statistics}). Note that the proposed estimator after the power iteration is given by
$$
\hat{\bT}^{\rm(iter_{\max})}=\hat{\bT}^{\rm init}\times_1 \bP_{U^{\rm(iter_{\max})}}\times_2 \bP_{V^{\rm(iter_{\max})}}\times_3 \bP_{W^{(\rm iter_{\max})}}
$$
where $U^{(\rm iter_{\max})}$ ($V^{\rm (iter_{\max})}, W^{\rm (iter_{\max})}$ resp.) denote the refined estimation after ${\rm iter}_{\max}=10$ power iterations. By including the estimator without power iteration, we can better appreciate the quality of spectral initialization and the effect of power iteration.

To further appreciate the merits of our approach, we also included an HOSVD based estimator:
$$
\hat{\bT}^{\rm (HOSVD)}=\hat{\bT}^{\rm init}\times_1 \bP_{\hat{U}^{\rm HOSVD}}\times_2 \bP_{\hat{V}^{\rm HOSVD}}\times_3 \bP_{\hat{W}^{\rm HOSVD}}.
$$
as well as the estimate proposed by \cite{montanari2016spectral}. We note that even though \cite{montanari2016spectral} considered only the noiseless case ($\sigma_\xi=0$), their estimator can nonetheless be applied to the noisy situations.

In our simulations, we set the sample size $n= r d^{\alpha}$ with various choices of $\alpha\in[0,3]$ and each observed entry is  perturbed with i.i.d. Gaussian noise $\xi\sim \mathcal{N}(0,\sigma_\xi^2)$. 
We set $d=50, 100, r=5$ and $\sigma_\xi=0.2$. For each $d, r, n$, all four estimates were evaluated based upon $30$ random realizations and the average error in estimating $\bT$:
$$
\varepsilon(\hat{\bT}):=\|\hat{\bT}-\bT\|_{\ell_2}/\|\bT\|_{\ell_2},
$$
and in estimating $U$
$$
\varepsilon(\hat U):=\|\hat U \hat U^{\top}-U U^\top\|
$$
are recorded. The results are summarized by Figures \ref{fig:spectral_third} and \ref{fig:fro_third}, where ``Naive'' represents HOSVD based estimator; ``MS'' stands for the estimator from \cite{montanari2016spectral}; ``U'' corresponds to $\hat{\bT}^{(0)}$ and ``U+Power'' our proposed estimator. The plots clearly show that $\hat{\bT}^{\rm (HOSVD)}$ requires a much larger sample size than the other estimates. It also suggests that $\hat{\bT}^{(0)}$ is more accurate than $\hat{\bT}^{\rm MS}$. Moreover, it shows that power iterations significantly improves the spectral estimation especially for larger $d$. Note that $\hat{\bT}^{\rm MS}$ can only be applied to $n\leq d^3$.

\begin{figure}
\centering
\subfigure[Comparison of spectral estimation between different approaches for $d=50,r=5$.]{\label{fig:spectral1}
  \includegraphics[height=3.3in,width=5in]{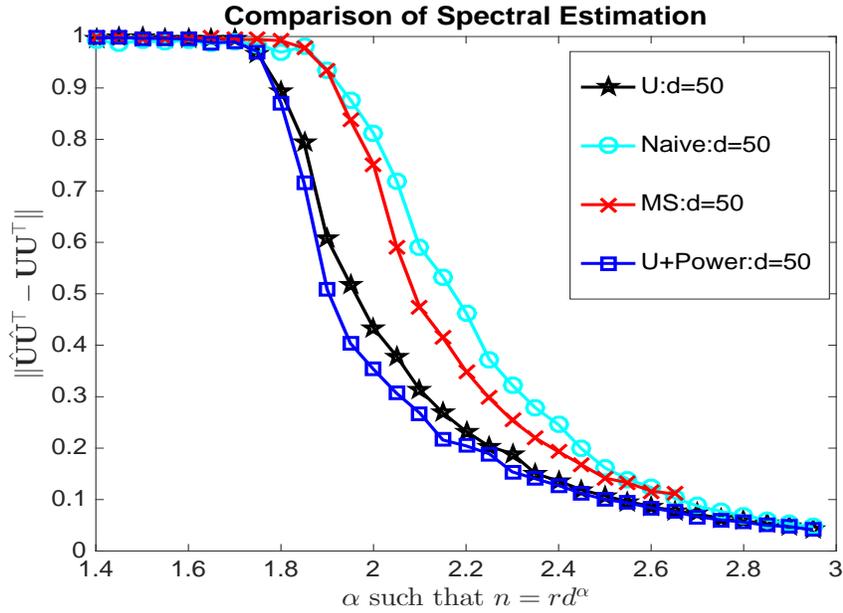}
}
\subfigure[Comparison of spectral estimation between different approaches for $d=100,r=5$.]{\label{fig:spectral2}
  \includegraphics[height=3.3in,width=5in]{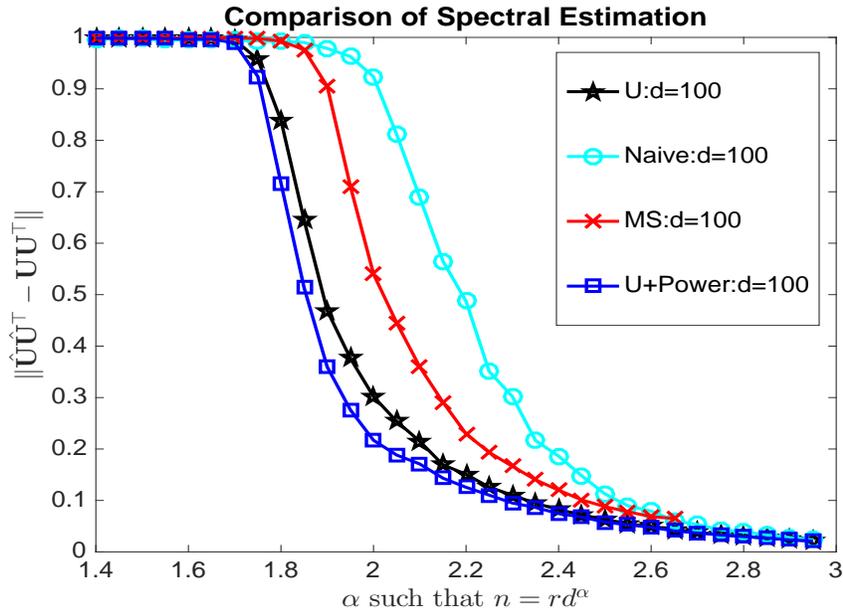}
  }
\caption{Comparison of spectral estimation among four different approaches. Note that ``MS'' method of \cite{montanari2016spectral} only applies to $n\leq d^3$.
}
\label{fig:spectral_third}
\end{figure}

\begin{figure}
\centering
\subfigure[Comparison of tensor recovery between different approaches for $d=50,r=5$.]{\label{fig:fro1}
  \includegraphics[height=3.3in,width=5in]{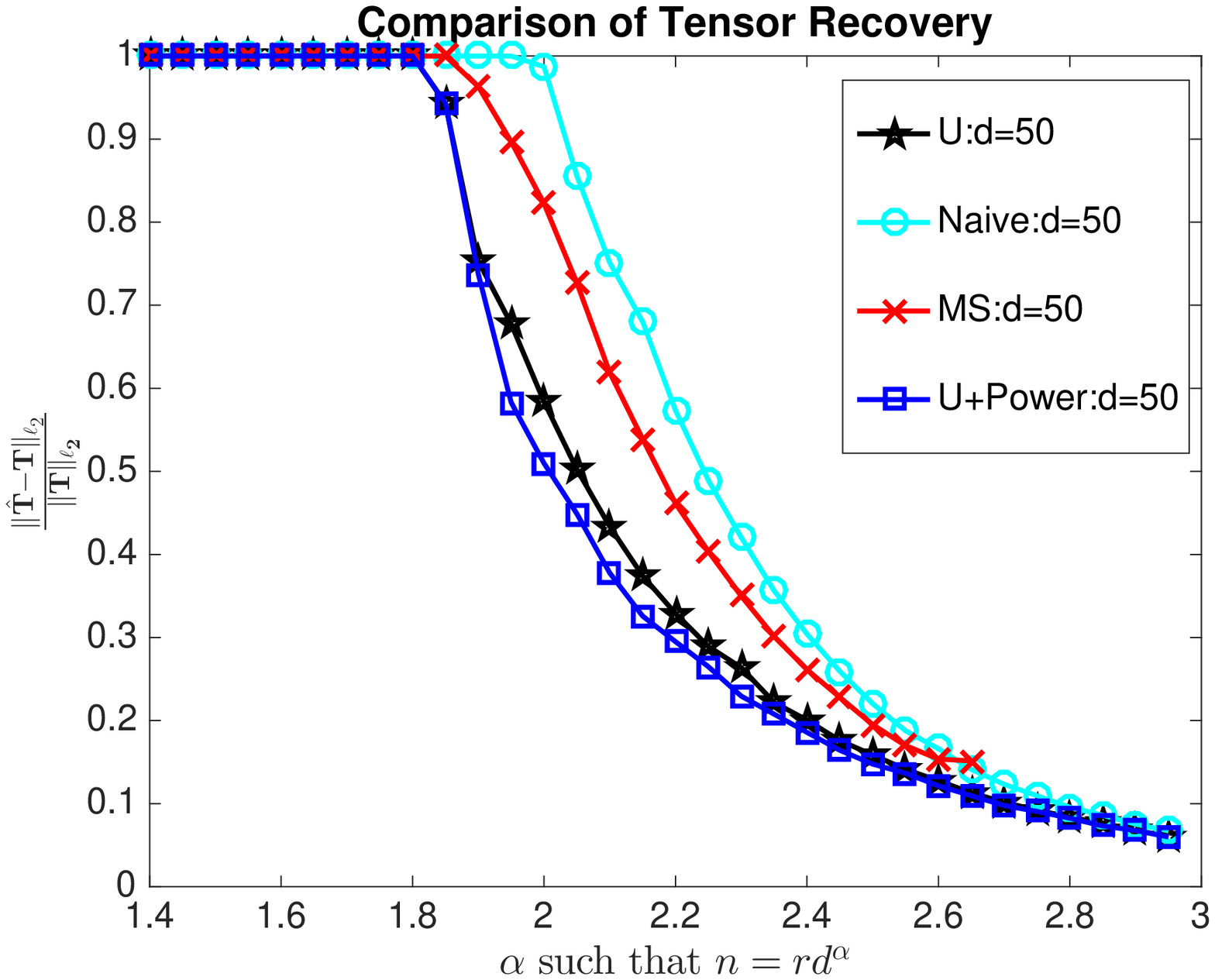}
}
\subfigure[Comparison of tensor recovery between different approaches for $d=100,r=5$.]{\label{fig:fro2}
  \includegraphics[height=3.3in,width=5in]{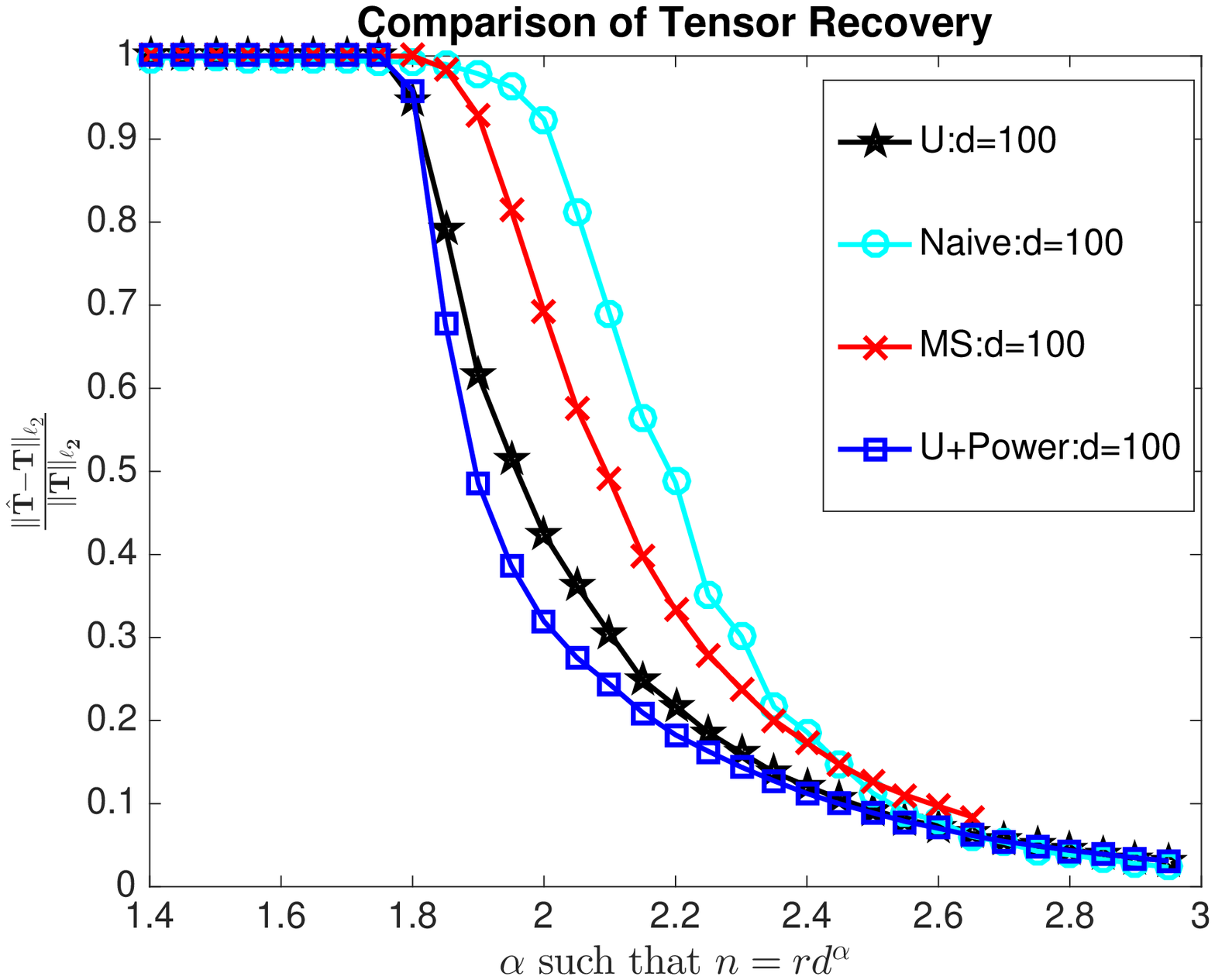}
  }
\caption{Comparison of tensor recovery among different approaches. Note that ``MS'' method of \cite{montanari2016spectral} only applies to $n\leq d^3$.
}
\label{fig:fro_third}
\end{figure}

Next, we apply our method to a simulated MRI brain image dataset to show the merits of our methods for denoising. The dataset can be accessed from McGill University Neurology Institute\footnote{http://brainweb.bic.mni.mcgill.ca/brainweb/}. See \cite{cocosco1997brainweb} and \cite{kwan1996extensible} for further details. We selected ``T1'' modality, ``1mm'' slice thickness, ``1\%'' noise, ``RF'' 40\% and obtained therefore a third-order tensor with size $217\times 181\times 181$, where each slice represents a $217\times 181$ brain image. The original tensor has full rank and we project it to a tensor with multilinear ranks $(20,20,20)$. In our simulations, we sampled $5\%,10\%,\ldots,100\%$ entries of $\bT$ and added i.i.d. Gaussian noise on each entries obeying distribution $\calN(0,\sigma_\xi^2)$ where
$$
\sigma_\xi=\gamma\cdot\bigg(\frac{\|\bT\|_{\ell_2}^2}{217\times 181\times 181}\bigg)^{1/2}
$$
with noise level $\gamma=0.05,0.10,0.15\ldots,1.0$. We applied our reconstruction scheme to each simulated dataset and recorded the relative error (RE): $\varepsilon(\hat\bT)=\|\hat\bT-\bT\|_{\ell_2}/\|\bT\|_{\ell_2}$. The results are presented in Figure~\ref{fig:MRI} and Figure~\ref{fig:MRI_2}. It again shows that our algorithm is quite stable to noise.
\begin{figure}
\centering
  \includegraphics[height=7.7in,width=6in]{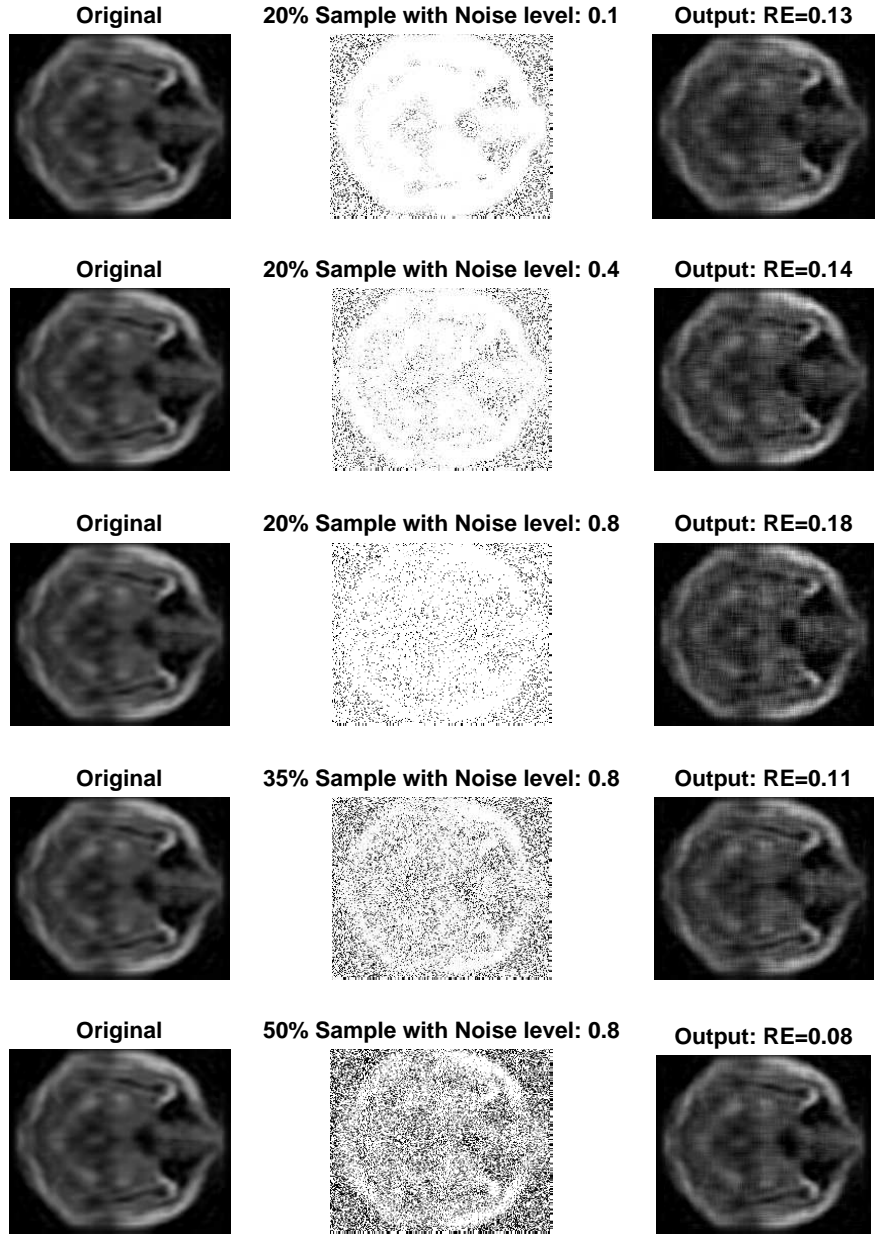}
\caption{Denoising of MRI brain image tensor. Each image is represents one slice of a tensor. The original tensor has size $217\times 181\times 181$ with multilinear ranks $(20,20,20)$. The third column represents the output of our algorithm with relative error (RE) measured as $\|\hat{\bT}-\bT\|_{\ell_2}/\|\bT\|_{\ell_2}$.}
\label{fig:MRI}
\end{figure}

\begin{figure}
\centering
  \includegraphics[height=3in,width=4in]{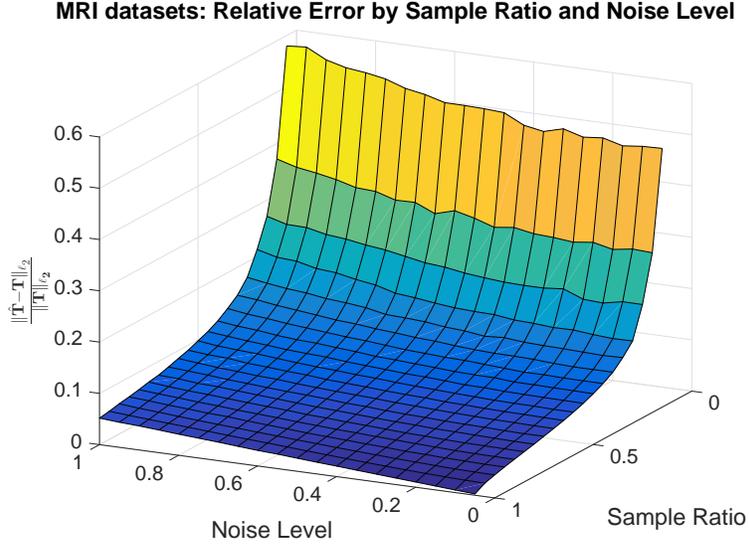}
\caption{Denoising of MRI brain image tensor. The dependence of relative error on the noise level and sample ratio. We observe that our algorithm is very stable to noise level.}
\label{fig:MRI_2}
\end{figure}

\section{Proofs}\label{sec:proofs}

We now present the proofs to our main results. We shall make use the Orlicz $\psi_{\alpha}$-norms ($\alpha\geq 1$) of a random variable $X$ defined as
$$
\|X\|_{\psi_\alpha}:=\inf\left\{u>0: \EE\exp\big(|X|^{\alpha}/u^{\alpha}\big)\leq 2\right\}.
$$
With this notion, the assumption that $\xi$ is subgaussian amounts to assuming that $\|\xi\|_{\psi_2}<+\infty$. A simple property of Orlicz norms that we shall use repeated without mentioning is the following: there exists a numerical constant $C>0$ such that for any random variables $X$ and $Y$, $\|XY\|_{\psi_1}\leq C\|X\|_{\psi_2}\|Y\|_{\psi_2}$ because
\begin{equation}\label{eq:orlicznorm}
\mathbb{E}\exp\left(\frac{|XY|}{ab}\right)\leq \mathbb{E}\exp\left(\frac{X^2}{2a^2}\right)\exp\left(\frac{Y^2}{2b^2}\right)\leq \mathbb{E}^{1/2}\exp\left(\frac{X^2}{a^2}\right)\mathbb{E}^{1/2}\exp\left(\frac{Y^2}{b^2}\right).
\end{equation}

\subsection{Proof of Theorem~\ref{th:init}}
The main architect of the proof follows a strategy developed by \cite{yuan2015tensor} for treating third order tensors. 

\paragraph{Symmetrization and Thinning.}
Let $\{\varepsilon_i\}_{i=1}^n$ denote i.i.d. Rademacher random variables independent with $\{(Y_i, \be_{\omega_i})\}_{i=1}^n$. Define
$$
\bDelta:=\frac{d_1\ldots d_k}{n}\sum_{i=1}^n\varepsilon_i Y_i \be_{\omega_i}.
$$
We begin with symmetrization (see, e.g., \cite{yuan2015tensor}) and obtain for any $t>0$,
\begin{eqnarray*}
\PP\Big(\Big\|\frac{d_1\ldots d_k}{n}\sum_{i=1}^nY_i\be_{\omega_i}-\bT\Big\|\geq t\Big)\leq 4\PP\Big(\|\bDelta\|\geq 2t\Big)\\
+4\exp\bigg(\frac{C_0t^2}{C_1d_1\ldots d_k(\sigma_{\xi}^2+\|\bT\|_{\ell_\infty}^2)/n+C_2t(\sigma_{\xi}+\|\bT\|_{\ell_\infty})d_1\ldots d_k/n}\bigg),
\end{eqnarray*}
for some universal constants $C_0,C_1,C_2>0$ and where we used Bernstein inequality of sum of independent subgaussian random variables. It suffices to prove the upper bound of $\PP\big(\|\bDelta\|\geq 2t\big)$ for any $t>0$.

Define
$$
\mathfrak{B}_{m_j,d_j}=\left\{0,\pm1,\pm2^{-1/2},\ldots,\pm2^{-m_j/2}\right\}^{d_j}\bigcap \left\{u\in\mathbb{R}^{d_j}: \|u\|_{\ell_2}\leq 1\right\},
$$
where $m_j=2\ceil{\log_2d_j}, j=1,\ldots,k$. 
As shown by \cite{yuan2015tensor},
$$
\|\bDelta\|=\sup_{\|u_j\|_{\ell_2}\le 1, j=1,\ldots,k} \langle \bDelta, u_1\otimes\cdots\otimes u_k\rangle\le 2^k\sup_{u_j\in \mathfrak{B}_{m_j,d_j}}\langle\bDelta, u_1\otimes\cdots\otimes u_k\rangle.
$$
In fact, we can take the supreme over an even smaller set on the rightmost hand side.

To this end, let $\bD_s$ be the operator that zeroes out the entries of tensor $\bA$ whose absolute value is not $2^{-s/2}$, that is
$$
\bD_s(\bA)=\sum_{a_1,\ldots,a_k}{\bf 1}\big\{\big|\langle\bA,e_{a_1}\otimes \ldots\otimes e_{a_k} \rangle\big|=2^{-s/2}\big\}\langle\bA,e_{a_1}\otimes \ldots\otimes e_{a_k} \rangle e_{a_1}\otimes \ldots\otimes e_{a_k},
$$
where, with slight abuse on the notations, we denote by $\{e_{a_j}: 1\le a_j\le d_j\}$ the canonical basis vectors in $\RR^{d_j}$ for $j=1,\ldots,k$. An essential observation is that the aspect ratio of the set $\Omega=\{\omega_i: 1\le i\le n\}$ is typically small. More specifically, write
$$
\nu_{\Omega}:=\max_{\ell=1,\ldots,k}\ \max_{a_j\in [d_j]: j\in [k]\setminus\ell}\  \big|\big\{a_{\ell}: (a_1,\ldots,a_k)\in\Omega\big\}\big|.
$$
It follows from Chernoff bound that there exists a constant $C>0$ such that for all $\alpha\geq 1$,
\begin{equation}\label{eq:nubound}
\nu_{\Omega}\leq C\alpha\max\left\{\frac{nd_{\max}}{d_1d_2\ldots d_k}, k\log d_{\max}\right\}=:\nu,
\end{equation}
with probability at least $1-d_{\max}^{-\alpha}$. See, e.g., \cite{yuan2016incoherent}. We shall now proceed conditional on this event.

Obviously,
$$
\sup_{u_j\in \mathfrak{B}_{m_j,d_j}}\langle\bDelta, u_1\otimes\cdots\otimes u_k\rangle=\sup_{u_j\in \mathfrak{B}_{m_j,d_j}}\langle\bDelta, \calP_{\Omega}\left(u_1\otimes\cdots\otimes u_k\right)\rangle,
$$
where $\calP_\Omega$ is the operator that zeroes all entries of a tensor outside $\Omega$. We shall now characterize $\calP_{\Omega}\bD_s(u_1\otimes \ldots\otimes u_k)$. For fixed $u_j\in\mathfrak{B}_{m_j,d_j}, j=1,\ldots,k$, write $A_{b_j}=\{a: |u_j(a)|=2^{-b_j/2}\}$. As shown in \cite{yuan2015tensor}, there exist sets $\tilde{A}_{s,b_j}\subset A_{b_j}$ such that
$$
|\widetilde{A}_{s,b_j}|^2\leq \nu_{\Omega}\left(\prod_{j=1}^{k}|\widetilde{A}_{s,b_j}|\right),
$$
$$
(A_{b_1}\times \ldots\times A_{b_k})\cap \Omega=(\widetilde{A}_{s,b_1}\times \ldots \times \widetilde{A}_{s,b_k})\cap \Omega,
$$
and
$$
\widetilde{\bD}_s(u_1\otimes \ldots\otimes u_k):=\calP_\Omega\widetilde{\bD}_s(u_1\otimes \ldots\otimes u_k)=\sum_{(b_1,\ldots,b_k):b_1+\ldots+b_k=s}\calP_{\widetilde{A}_{s,b_1}\times\ldots\times \widetilde{A}_{s,b_k}}\bD_s(u_1\otimes \ldots\otimes u_k).
$$
Now define
$$
\mathfrak{B}^{\star}_{\Omega,m_{\star}}:=\Big\{\sum_{0\leq s\leq m_{\star}}\widetilde{\bD}_s(u_1\otimes \ldots\otimes u_k)+\sum_{m_{\star}<s\leq m^{\star}}\bD_s(u_1\otimes \ldots\otimes u_k): u_{j}\in \mathfrak{B}_{m_j,d_j}, j=1,\ldots,k\Big\}
$$
for any $0\leq m_{\star}\leq m^{\star}=\sum_{j=1}^k m_j$. Write
$$\mathfrak{B}^{\star}_{\nu,m_{\star}}=\bigcup_{\nu_{\Omega}\leq \nu}\mathfrak{B}^{\star}_{\Omega,m_{\star}}.$$
Then
$$
\|\bDelta\|\leq 2^{k}\underset{\bY\in\mathfrak{B}^{\star}_{\nu,m_{\star}}}{\max} \big\langle\bY,\bDelta\big\rangle.
$$
It is not hard to see that \citep{yuan2015tensor}
$$
\log{\rm Card}\big(\mathfrak{B}_{\nu,m_{\star}}^{\star}\big)\leq \frac{21}{4}(d_1+d_2+\ldots+d_k).
$$
A refined characterization of the entropy of $\mathfrak{B}_{\nu,m_{\star}}^{\star}$ is also needed. To this end, define for any $0\leq q\leq s\leq m_{\star}$,
$$
\mathfrak{D}_{v,s,q}:=\Big\{\bD_s(\bY): \bY\in\mathfrak{B}^{\star}_{v,m_{\star}}, \|\bD_s(\bY)\|_{\ell_2}^2\leq 2^{q-s}\Big\}.
$$
Following an identical argument to Lemma 12 of \cite{yuan2015tensor}, we have (for readers' convenience, we include its proof in the Appendix for completeness.)
\begin{lemma}\label{lemma:Dvjk}
Let $\nu\geq 1$. For all $0\leq q\leq s\leq m_{\star}$, the following bound holds
$$
\log{\rm Card}(\mathfrak{D}_{\nu,s,q})\leq qs^k\log 2+2k^2s^{k}\sqrt{\nu 2^q} L\big(\sqrt{\nu 2^q}, d_{\max}s^{k/2}\big),
$$
where $L(x,y)=\max\big\{1, \log(ey/x)\big\}$.
\end{lemma}

We are now in position to bound $\|\bDelta\|$. Observe that
\begin{eqnarray*}
\big\|\bDelta\big\|&\leq& 2^k\underset{\bY\in\mathfrak{B}^{\star}_{\nu,m_{\star}}}{\max} \left\langle\bY,\quad \bDelta\right\rangle\\
&=&2^{k}\underset{\bY\in\mathfrak{B}^{\star}_{v,m_{\star}}}{\max} \bigg(\sum_{0\leq s\leq m_{\star}}\Big<\bD_s(\bY),\quad \bDelta\Big>+\big<\bS_{\star}(\bY), \bDelta\big>\bigg),
\end{eqnarray*}
where $\bS_{\star}(\bY)=\sum_{s>m_{\star}}\bD_s(\bY)$ and $m_{\star}$ is determined by
$$
m_{\star}:=\min\big\{x: x\geq m^{\star}\ {\rm or}\ 2k^2x^{k}\sqrt{\nu 2^x}L\big(\sqrt{\nu 2^x},d_{\max}x^{k/2}\big)\geq d_1+\ldots+d_k\big\}.
$$
Another simple fact is that $m_{\star}\leq m^{\star}\lesssim k\ceil{\log (d_{\max})}$.

\paragraph{Bounding $\big|\langle\bD_s(\bY), \bDelta\rangle\big|$.} For any $\bY\in \mathfrak{B}^{\star}_{v,m_{\star}}$, we have $2^{-s}\leq \|\bD_s(\bY)\|_{\rm F}^2\leq 1$ and thus $\bD_s(\bY)\in \cup_{q=0}^{s}\mathfrak{D}_{v,s,q}$. Denote $\bY_s=\bD_s(\bY)$.
It suffices to develop an upper bound for
$$
\underset{\bY_s\in \mathfrak{D}_{\nu,s,q}\setminus \mathfrak{D}_{\nu,s,q-1}}{\max}\quad \langle\bY_s, \bDelta\rangle=\sum_{i=1}^n\langle\bY_s, \bZ_i\rangle,
$$
for all $0\leq q\leq s$, where
$$\bZ_i:=\frac{d_1\ldots d_k}{n}\varepsilon_iY_i\be_{\omega_i},\qquad i=1,\ldots, n.$$
Observe that, for any fixed $\bY_s\in\mathfrak{D}_{v,s,q}\setminus \mathfrak{D}_{v,s,q-1}$,
\begin{eqnarray*}
\mathbb{E}\langle\bY_s,\bZ_i \rangle^2&\leq& 2\frac{(d_1\ldots d_k)^2}{n^2}\mathbb{E}\langle\bT, \be_{\omega_i}\rangle^2\langle\bY_s,\be_{\omega_i} \rangle^2+2\frac{(d_1\ldots d_k)^2}{n^2}\mathbb{E}\xi^2\langle\be_{\omega_i}, \bY_{s}\rangle^2\\
&\leq& 2\frac{d_1\ldots d_k}{n^2}\big(\|\bT\|_{\ell_\infty}^2+\sigma_{\xi}^2\big)\|\bY_s\|_{\rm \ell_2}^2\leq 2\frac{d_1\ldots d_k}{n^2}\big(\|\bT\|_{\ell_\infty}^2+\sigma_{\xi}^2\big)2^{q-s},
\end{eqnarray*}
and
\begin{eqnarray*}
\big\|\langle\bY_s, \bZ_i\rangle\big\|_{\psi_2}&\leq& \left\|\frac{d_1\ldots d_k}{n}\varepsilon_i\langle\be_{\omega_i}, \bT\rangle\langle\bY_s,\be_{\omega_i} \rangle\right\|_{\psi_2}+\left\|\frac{d_1\ldots d_k}{n}\varepsilon_i\xi_i\langle\be_{\omega_i},\bY_s \rangle\right\|_{\psi_2}\\
&\leq& C\frac{d_1\ldots d_k}{n}\big(\|\bT\|_{\ell_\infty}+\sigma_{\xi}\big)2^{-s/2},
\end{eqnarray*}
for some constant $C>0$, implying that $\langle\bY_s,\bZ_i\rangle$ has a subgaussian tail. By Bernstein inequality for sum of unbounded random variables,
\begin{align*}
&\mathbb{P}\bigg(\Big|\sum_{i=1}^n\langle\bY_s, \bZ_i\rangle\Big|\geq t\bigg)\\
\leq& \exp\bigg(-\frac{C_0t^2}{C_1d_1\ldots d_k(\sigma_{\xi}^2+\|\bT\|_{\ell_\infty}^2)2^{q-s}/n+C_2t(\sigma_{\xi}+\|\bT\|_{\ell_\infty})d_1\ldots d_k2^{-s/2}/n}\bigg),
\end{align*}
for some universal constants $C_0, C_1, C_2>0$. An application of the union bound yields
\begin{eqnarray*}
&&\mathbb{P}\bigg(\max_{\bY_s\in\mathfrak{D}_{v,s,q}\setminus \mathfrak{D}_{v,s,q-1}} \Big|\sum_{i=1}^n\langle\bY_s, \bZ_i\rangle\Big|\geq t\bigg)\\
&\leq& \big|\mathfrak{D}_{v,s,q}\big|\exp\bigg(-\frac{C_0t^2}{C_1d_1\ldots d_k(\|\bT\|_{\ell_\infty}\vee\sigma_{\xi})^22^{q-s}/n+C_2t(\|\bT\|_{\ell_\infty}\vee\sigma_{\xi})d_1\ldots d_k2^{-s/2}/n}\bigg)\\
&\leq& \exp\bigg(\frac{21}{4}(d_1+\ldots+d_k)-\frac{C_0t^2}{C_1d_1\ldots d_k(\|\bT\|_{\ell_\infty}\vee\sigma_{\xi})^22^{q-s}/n}\bigg)\\
&&\qquad +\exp\bigg(\log{\rm Card}(\mathfrak{D}_{v,s,q})-2^{s/2}\frac{C_0t}{C_2(\|\bT\|_{\ell_\infty}\vee\sigma_{\xi})d_1\ldots d_k/n}\bigg).
\end{eqnarray*}
Recall that $m_{\star}\lesssim k\log (d_{\max})$,
$$
\log{\rm Card}(\mathfrak{D}_{v,s,q})\lesssim (k\log d_{\max})^{k+1}+2k^2(k\log d_{\max})^k\sqrt{\nu 2^{q}}L(\sqrt{\nu2^{q}}, d_{\max}s^{k/2}),
$$
and
$$L(\sqrt{\nu 2^q},d_{\max}s^{k/2})\lesssim k\log d_{\max}.$$
By choosing
\begin{align}\nonumber
t\geq C_1(\|\bT\|_{\ell_\infty}\vee \sigma_{\xi})\max\bigg\{2^{(q-s)/2}\sqrt{\frac{kd_{\max}d_1\ldots d_k}{n}},2^{-s/2}(k\log d_{\max})^{k+1}\frac{d_1\ldots d_k}{n},\\ 
k^3(k\log d_{\max})^{k}\sqrt{v}2^{(q-s)/2}\frac{d_1\ldots d_k\log d_{\max}}{n}\bigg\},\label{eq:tbound}
\end{align}
we get
\begin{eqnarray*}
\mathbb{P}\bigg(\max_{\bY_s\in\mathfrak{D}_{v,s,q}\setminus \mathfrak{D}_{v,s,q-1}} \Big|\sum_{i=1}^n\langle\bY_s, \bZ_i\rangle\Big|\geq t\bigg)&\leq& \exp\bigg(-\frac{C_0t^2}{C_1d_1\ldots d_k(\|\bT\|_{\ell_\infty}\vee\sigma_{\xi})^22^{q-s}/n}\bigg)+\\
&&\qquad +\exp\bigg(-2^{s/2}\frac{C_0t}{C_2(\|\bT\|_{\ell_\infty}\vee\sigma_{\xi})d_1\ldots d_k/n}\bigg).
\end{eqnarray*}
By making the above bound uniform over all pairs $0\leq q\leq s\leq m_{\star}$, we obtain that 
\begin{eqnarray*}
&&\mathbb{P}\bigg(\max_{\bY\in\mathfrak{B}^{\star}_{\nu,m_{\star}}} \Big|\sum_{0\leq s\leq m_{\star}}\sum_{i=1}^n\langle\bY_s, \bZ_i\rangle\Big|\geq (m_{\star}+1)t\bigg)\\
&\leq& 1-{m_{\star}+2\choose 2}\exp\bigg(-\frac{C_0t^2}{C_1d_1\ldots d_k(\|\bT\|_{\ell_\infty}\vee\sigma_{\xi})^22^{q-s}/n}\bigg)\\
&&\qquad -{m_{\star}+2\choose 2}\exp\bigg(-2^{s/2}\frac{C_0t}{C_2(\|\bT\|_{\ell_\infty}\vee\sigma_{\xi})d_1\ldots d_k/n}\bigg).
\end{eqnarray*}

\paragraph{Bounding $\max_{\bY\in \mathfrak{B}^{\star}_{\nu, m_{\star}}}\big|\sum_{i=1}^n \langle\bS_{\star}(\bY),\bZ_i \rangle\big|$.}
Observe that
\begin{eqnarray*}
\mathbb{E}\langle\bS_{\star}(\bY),\bZ_i \rangle^2\leq 2\frac{(d_1\ldots d_k)^2}{n^2}\mathbb{E}\langle\bS_{\star}(\bY),\be_{\omega_i} \rangle^2\langle\bT,\be_{\omega_i} \rangle^2+2\frac{(d_1\ldots d_k)^2}{n^2}\mathbb{E}\xi_i^2\langle\bS_{\star}(\bY),\be_{\omega_i} \rangle^2\\
\leq 2\frac{d_1\ldots d_k}{n^2}\|\bS_{\star}(\bY)\|_{\rm F}^2\big(\|\bT\|_{\ell_\infty}^2+\sigma_{\xi}^2\big)\leq 2^{-m_{\star}+1}\frac{d_1\ldots d_k}{n^2}\big(\|\bT\|_{\ell_\infty}^2+\sigma_{\xi}^2\big),
\end{eqnarray*}
and
\begin{align*}
\big\|\langle\bS_{\star}(\bY), \bZ_i\rangle\big\|_{\psi_2}\leq& \big\|\frac{d_1\ldots d_k}{n}\varepsilon_i\langle\bS_{\star}, \be_{\omega_i}\rangle\langle\bT,\be_{\omega_i} \rangle\big\|_{\psi_2}+\big\|\frac{d_1\ldots d_k}{n}\varepsilon_i\xi_i\langle\bS_{\star}(\bY), \be_{\omega_i}\rangle\big\|_{\psi_2}\\
\leq& C\frac{d_1\ldots d_k}{n}2^{-m_{\star}/2}\big(\|\bT\|_{\ell_\infty}+\sigma_{\xi}\big),
\end{align*}
for some constant $C>0$. Again, by Bernstein inequality and the union bound,
\begin{align*}
\mathbb{P}\bigg(\Big|&\max_{\bY\in\mathfrak{B}^{\star}_{v,m_{\star}}}\sum_{i=1}^{n}\langle\bS_{\star}(\bY),\bZ_i \rangle\Big|\geq t\bigg)\leq \exp\Big(21/4(d_1+\ldots+d_k)\Big)\\
 \times&\exp\Big(-\frac{C_0t^2}{C_1d_1\ldots d_k2^{-m_{\star}+1}(\|\bT\|_{\ell_\infty}\vee \sigma_{\xi})^2/n+C_2t(\|\bT\|_{\ell_\infty}\vee\sigma_{\xi})d_1\ldots d_k2^{-m_{\star}/2}/n}\Big).
\end{align*}
By choosing 
$$
t\geq C(\|\bT\|_{\ell_\infty}\vee\sigma_{\xi})\max\bigg\{2^{-(m_{\star}-1)/2}\sqrt{\frac{kd_{\max}d_1\ldots d_k}{n}}, \quad 2^{-m_{\star}/2}\frac{kd_{\max}d_1\ldots d_k}{n}\bigg\},
$$
we get
\begin{eqnarray*}
\mathbb{P}\bigg(\Big|\max_{\bY\in\mathfrak{B}^{\star}_{v,m_{\star}}}\sum_{i=1}^{n}\langle\bS_{\star}(\bY),\bZ_i \rangle\Big|\geq t\bigg)
\leq \exp\bigg(-\frac{C_0t^2}{C_1d_1\ldots d_k2^{-m_{\star}+1}(\|\bT\|_{\ell_\infty}\vee\sigma_{\xi})^2/n}\bigg)\\
+\exp\bigg(-\frac{C_0t}{C_2(\|\bT\|_{\ell_\infty}\vee\sigma_{\xi})d_1\ldots d_k2^{-m_{\star}/2}/n}\bigg).
\end{eqnarray*}

\paragraph{Putting them together.} Combining above bounds, we conclude that if
\begin{eqnarray*}
t\geq C_1(\|\bT\|_{\ell_\infty}\vee \sigma_{\xi})\max\bigg\{\sqrt{\frac{kd_{\max}d_1\ldots d_k}{n}}, \quad (k\log d_{\max})^{k+1}\frac{d_1\ldots d_k}{n},\\
k^3(k\log d_{\max})^{k}\sqrt{v}\frac{d_1\ldots d_k\log d_{\max}}{n},\quad 2^{-m_{\star}/2}\frac{kd_{\max}d_1\ldots d_k}{n}\bigg\},
\end{eqnarray*}
then
\begin{eqnarray*}
\mathbb{P}\bigg(\|\bDelta\|\leq (m_{\star}+2)t\bigg)
\geq1-2{m_{\star}+2\choose 2}\exp\bigg(-\frac{C_0t^2}{C_1d_1\ldots d_k(\|\bT\|_{\ell_\infty}\vee\sigma_{\xi})^2/n}\bigg)\\
-2{m_{\star}+2\choose 2}\exp\bigg(-\frac{C_0t}{C_2(\|\bT\|_{\ell_\infty}\vee\sigma_{\xi})d_1\ldots d_k/n}\bigg).
\end{eqnarray*}
By the definition of $m_{\star}$, we have
$$2^{-m_{\star}/2}\lesssim \frac{\sqrt{\nu}}{d_{\max}}k^{3+k}\log^{k+1} d_{\max}.$$
Therefore, with probability at least $1-d_{\max}^{-\alpha}$ for $\alpha>1$ (by adjusting the constant $C_1$),
\begin{eqnarray*}
\big\|\bDelta\big\|\leq C_1k^{k+3}\alpha\big(\|\bT\|_{\ell_\infty}\vee \sigma_{\xi}\big)\max\bigg\{\sqrt{\frac{kd_{\max}d_1\ldots d_k}{n}},\quad \frac{kd_1\ldots d_k}{n} \bigg\}\log^{k+2}d_{\max}.
\end{eqnarray*}
Similar bounds also hold for $\Big\|n^{-1}d_1\ldots d_k\sum_{i=1}^nY_i\be_{\omega_i}-\bT\Big\|$. 

\subsection{Proof of Theorem~\ref{thm:N1bound}}
Without loss of generality, we only consider $j=1$ and prove the upper bound of $\|\hat{\bN}_1-\bN_1\|$.
Recall that $\hat\bN_1$ can be equivalently written as
$$
\hat\bN_1=\frac{(d_1\ldots d_k)^2}{n(n-1)}\sum_{1\leq i< i'\leq n}Y_iY_{i'}\left(\calM_1(\be_{\omega_i})\calM_1^{\top}(\be_{\omega_{i'}})+\calM_1(\be_{\omega_{i'}})\calM_1^{\top}(\be_{\omega_i})\right).
$$
Note that $\hat\bN_1$ is actually a U-statistics. By a standard decoupling technique \citep[see, e.g.,][]{de1999decoupling}, 
\begin{equation}
\mathbb{P}\big(\|\hat\bN_1-\bN_1\|\geq t\big)\leq 15\mathbb{P}\big(\|\widetilde\bN_1-\bN_1\|\geq 15t\big),
\end{equation}
where
$$
\widetilde\bN_1:=\frac{(d_1\ldots d_k)^2}{2n(n-1)}\sum_{1\leq i\neq i'\leq n}Y_i\widetilde Y_{i'}\left(\calM_1(\be_{\omega_i})\calM_1^{\top}(\widetilde\be_{\omega_{i'}})+\calM_1(\widetilde\be_{\omega_{i'}})\calM_1^{\top}(\be_{\omega_i})\right),
$$
and $\{(\widetilde\be_{\omega_i},\widetilde Y_i): 1\le i\le n\}$ is an independent copy of $\{(\be_{\omega_i},Y_i): 1\le i\le n\}$ such that
$$
\widetilde{Y}_i=\langle \bT,\widetilde{\be}_{\omega_i}\rangle+\widetilde{\xi}_i,\qquad i=1,\ldots,n.
$$
For simplicity, let $m_1=d_1, m_2=\frac{d_1\ldots d_k}{d_1}$ and denote $\bM=\calM_1(\bT)\in\mathbb{R}^{m_1\times m_2}$. With slight abuse of notation, write $\bX_i=\calM_1(\be_{\omega_i})$ and $\widetilde{\bX}_i=\calM_1(\widetilde{\be}_{\omega_i})$. Define
$$
\bS_1:=\bDelta_1+{\bf \Xi}_1:=\Big(\frac{m_1m_2}{n}\sum_{i=1}^n Y_i\bX_i\Big)-\bM\quad {\rm and}\quad
\bS_2:=\bDelta_2+{\bf \Xi}_2:=\Big(\frac{m_1m_2}{n}\sum_{i=1}^n \widetilde{Y}_i\widetilde{\bX}_i\Big)-\bM,
$$
where
$$
\bDelta_1=\Big(\frac{m_1m_2}{n}\sum_{i=1}^n\langle\bM,\bX_i \rangle\bX_i\Big)-\bM\quad {\rm and}\quad \bDelta_2=\Big(\frac{m_1m_2}{n}\sum_{i'=1}^n\langle\bM,\widetilde{\bX}_{i'}\rangle\widetilde{\bX}_{i'}\Big)-\bM,
$$
and
$$
{\bf \Xi}_1=\frac{m_1m_2}{n}\sum_{i=1}^n\xi_i\bX_i\quad {\rm and}\quad {\bf \Xi}_2=\frac{m_1m_2}{n}\sum_{i'=1}^n\widetilde{\xi}_{i'}\widetilde{\bX}_{i'}.
$$
Write
\begin{eqnarray*}
\widetilde{\bN}_1-\bN_1=\frac{n}{2(n-1)}\big(\bS_1\bS_2^{\top}+\bS_2\bS_1^{\top}\big)+\frac{n}{2(n-1)}\big(\bS_1+\bS_2\big)\bM^{\top}+\frac{n}{2(n-1)}\bM(\bS_1^{\top}+\bS_2^{\top})\\
+\frac{1}{n-1}\Big(\frac{(m_1m_2)^2}{2n}\sum_{i=1}^nY_i\widetilde{Y}_i\big(\bX_i\widetilde{\bX}_i^{\top}+\widetilde{\bX}_i\bX_i^{\top}\big)-\bM\bM^{\top}\Big).
\end{eqnarray*}
We now bound the spectral norm of each term on the righthand side separately. We begin with several preliminary facts which can be easily proved by matrix Bernstein inequalities (Lemma~\ref{lemma:matBern-bounded} and Lemma~\ref{lemma:matBern-unbound}).
By Lemma~\ref{lemma:matBern-unbound}, with probability at least $1-m_{\max}^{-\alpha}$ for any $\alpha\geq 1$, 
\begin{eqnarray*}
\max\big\{\|{\bf \Xi}_1\|,\|{\bf \Xi}_2\|\big\}\leq C\sigma_{\xi}\max\bigg\{\sqrt{\frac{\alpha m_1m_2m_{\max}\log m_{\max}}{n}},\frac{\alpha m_1m_2\log^{3/2}m_{\max}}{n}\bigg\},
\end{eqnarray*}
where $m_{\max}=\max\{m_1,m_2\}$, $m_{\min}=\min\{m_1,m_2\}$ and $C>0$ is an absolute constant. Denote this event by $\calE_1$. By Lemma~\ref{lemma:matBern-bounded}, the following bound holds with probability at least $1-m_{\max}^{-\alpha}$ for any $\alpha\geq 1$,
$$
\max\big\{\|\bDelta_1\|,\|\bDelta_2\|\big\}\leq C\|\bM\|_{\ell_\infty}\max\bigg\{\sqrt{\frac{\alpha m_1m_2m_{\max}\log m_{\max}}{n}}+\frac{\alpha m_1m_2\log m_{\max}}{n}\bigg\}.
$$
Denote this event by $\calE_2$. By Chernoff bound \citep[see, e.g.][]{yuan2015tensor}, there exists an event $\calE_3$ with $\mathbb{P}(\calE_3)\geq 1-m_{\max}^{-\alpha}-n^{-\alpha}$ such that on $\calE_3$,
$$
\max\big\{\|\bDelta_1\|_{1,\infty},\|\bDelta_2\|_{1,\infty}\big\}\leq m_1\|\bM\|_{\ell_\infty}+\frac{(3\alpha+7)(m_1m_2)}{n}\|\bM\|_{\ell_\infty}\Big(\frac{n}{m_2}+\log m_{\max}\Big),
$$
and
$$
\max\big\{\|{\bf \Xi}_1\|_{1,\infty},\|{\bf \Xi}_2\|_{1,\infty}\big\}\leq C\frac{\alpha(3\alpha+7)(m_1m_2)}{n}\sigma_{\xi}\Big(\frac{n}{m_2}+\log m_{\max}\Big)\log^{1/2}n,
$$
for any $\alpha\geq 1$, where 
$$
\|\bA\|_{1,\infty}:=\max_{1\leq j\leq m_2}\sum_{i=1}^{m_1}\big|\bA_{ij}\big|,\quad\forall \bA\in\mathbb{R}^{m_1\times m_2}.
$$
Here, we used the fact that \citep[see, e.g.,][]{vaart1997weak}
$$
\mathbb{P}\Big(\max_{1\leq i\leq n}|\xi_i|\geq C\alpha\sigma_{\xi}\log^{1/2}n\Big)\leq n^{-\alpha}.
$$

We first bound $\|\bS_1\bS_2^{\top}\|$ and $\|\bS_2\bS_1^{\top}\|$ which can be treated in an identical fashion. We shall consider $\|\bS_2\bS_1^{\top}\|$ only for brevity. We proceed conditional on the event $\calE_1\cap\calE_2\cap \calE_3$.
$$
\bS_2\bS_1^{\top}=\Big(\frac{m_1m_2}{n}\sum_{i=1}^n\langle\bM,\bX_i\rangle\bS_2\bX_i^{\top}-\bS_2\bM^{\top}\Big)+\Big(\frac{m_1m_2}{n}\sum_{i=1}^n\xi_i\bS_2\bX_i^{\top}\Big).
$$
For any fixed $\bS_2$, we can apply matrix Bernstein inequality to control $\big\|n^{-1}(m_1m_2)\sum_{i=1}^n\xi_i\bS_2\bX_i^{\top}\big\|$. Observe that
$$
\big\|\|\xi\bS_2\bX^{\top}\|\big\|_{\psi_2}\leq \|\bS_2\bX^{\top}\|\|\xi\|_{\psi_2}\lesssim \sigma_{\xi}\|\bS_2\bX^{\top}\|.
$$
We can further bound $\|\bS_2\bX^{\top}\|$ by
\begin{eqnarray*}
\|\bS_2\bX^{\top}\|\leq &\|\bS_2\|_{1,\infty}\leq& \|\bDelta_2\|_{1,\infty}+\|{\bf \Xi}_2\|_{1,\infty}\\
&&\leq m_1\|\bM\|_{\ell_\infty}+\frac{(3\alpha+7)m_1m_2}{n}\|\bM\|_{\ell_\infty}\Big(\frac{n}{m_2}+\log m_{\max}\Big)\\
&&\qquad +C\sigma_{\xi}\frac{\alpha(3\alpha+7)m_1m_2}{n}\Big(\frac{n}{m_2}+\log m_{\max}\Big)\log^{1/2}n.
\end{eqnarray*}
Moreover, 
\begin{eqnarray*}
\sigma_{\xi}^2m_2^{-1}\|\bS_2\|^2&\leq&
\max\Big\{\big\|\mathbb{E}\xi^2\bS_2\bX^{\top}\bX\bS_2^{\top}\big\|,\big\|\mathbb{E}\xi^2\bX\bS_2^{\top}\bS_2\bX^{\top}\big\|\Big\}\\
&\leq& \sigma_{\xi}^2m_2^{-1}\|\bS_2\|^2+\sigma_{\xi}^2(m_1m_2)^{-1}\tr(\bS_2^{\top}\bS_2)\leq 2\sigma_{\xi}^2m_2^{-1}\|\bS_2\|^2.
\end{eqnarray*}
By matrix Bernstein inequality, with probability at least $1-m_{\max}^{-\alpha}$,
\begin{align*}
&\Big\|\frac{m_1m_2}{n}\sum_{i=1}^n\xi_i\bS_2\bX_i^{\top}\Big\|\\
\leq&C\sigma_{\xi}\|\bS_2\|\sqrt{\frac{\alpha m_1^2m_2\log m_{\max}}{n}}+C\sigma_{\xi}\|\bS_2\|_{\ell_\infty}\frac{\alpha m_1m_2\log m_{\max}}{n}\log\Big(\frac{\sqrt{m_2}\|\bS_2\|_{\ell_{\infty}}}{\|\bS_2\|}\Big).
\end{align*}
Denote this event by $\calE_4$. In a similar fashion, we can apply matrix Bernstein inequality to bound
$$
\Big\|\frac{m_1m_2}{n}\sum_{i=1}^n\bS_2\Big(\langle\bM,\bX_i \rangle\bX_i^{\top}\Big)-\bM^{\top}\Big\|.
$$
Clearly,
$$
\Big\|\bS_2\Big(m_1m_2\langle\bM,\bX \rangle\bX^{\top}-\bM^{\top}\Big)\Big\|\leq m_1m_2\|\bM\|_{\ell_\infty}\|\bS_2\|_{\ell_\infty}+\|\bS_2\|\|\bM\|.
$$
Moreover,
\begin{eqnarray*}
\max\Big\{\Big\|\mathbb{E}\bS_2\langle\bM,\bX \rangle^2\bX^{\top}\bX\bS_2^{\top}\Big\|, \Big\|\mathbb{E}\langle\bM,\bX \rangle^2\bX\bS_2^{\top}\bS_2\bX^{\top}\Big\|\Big\}\leq m_2^{-1}\|\bS_2\|^2\|\bM\|_{\ell_\infty}^2.
\end{eqnarray*}
By matrix Bernstein inequality, with probability at least $1-m_{\max}^{-\alpha}$,
\begin{eqnarray*}
&&\Big\|\frac{1}{n}\sum_{i=1}^n\bS_2\Big(m_1m_2\langle\bM,\bX_i \rangle\bX_i^{\top}-\bM^{\top}\Big)\Big\|\\
&\leq& C\|\bM\|_{\ell_\infty}\|\bS_2\|\sqrt{\frac{\alpha m_1^2m_2\log m_{\max}}{n}}+C\Big(m_1m_2\|\bM\|_{\ell_\infty}\|\bS_2\|_{1,\infty}+\|\bS_2\|\|\bM\|\Big)\frac{\alpha\log m_{\max}}{n}.
\end{eqnarray*}
Denote the above event by $\calE_5$. We conclude that, conditional on the event $\bigcap_{k=1}^5\calE_k$,
\begin{eqnarray*}
\|\bS_2\bS_1^{\top}\|\leq C\bigg(\|\bM\|_{\ell_\infty}\sigma_{\xi}\frac{\alpha^2(3\alpha+7)(m_1^3m_2^2m_{\max})^{1/2}\log^{3/2}m_{\max}\log^{1/2}n}{n}\\
+\sigma_{\xi}^2\frac{\alpha^2(3\alpha+7)(m_1^3m_2^2m_{\max})^{1/2}\log^{3/2}m_{\max}\log^{1/2}n}{n}\\
+\alpha(3\alpha+7)\|\bM\|_{\ell_\infty}^2\frac{(m_1^3m_2^2m_{\max})^{1/2}\log m_{\max}}{n}\bigg).
\end{eqnarray*}
Here, to simplify the bounds, we assumed that $n\geq C\alpha(\sqrt{m_1m_2}\log m_{\max}+m_1\log^2 m_{\max})$.

Next we bound $\|(\bS_1+\bS_2)\bM^{\top}\|$. To fix ideas, we only deal with $\bS_1\bM^{\top}$ which can be written as
$$
\bS_1\bM^{\top}=\bDelta_1\bM^{\top}+{\bf \Xi}_1\bM^{\top}.
$$
Clearly, they can be each controlled by matrix Bernstein inequalities in an identical fashion as above. Indeed, we can write
$$
\bDelta_1\bM^{\top}=\frac{1}{n}\sum_{i=1}^n\Big(m_1m_2\langle\bM,\bX_i\rangle\bX_i-\bM\Big)\bM^{\top},
$$
and
$$
{\bf \Xi}_1\bM^{\top}=\frac{1}{n}\sum_{i=1}^nm_1m_2\xi_i\bX_i\bM^{\top}.
$$
By Lemma~\ref{lemma:matBern-bounded} and Lemma~\ref{lemma:matBern-unbound}, we obtain that with probability at least $1-m_{\max}^{-\alpha}$,
\begin{eqnarray*}
\max\{\|\bS_1\bM^{\top}\|, \|\bS_2\bM^{\top}\|\}\leq C\|\bM\|_{\ell_\infty}\|\bM\|\Big(\sqrt{\frac{\alpha m_1^2m_2\log m_{\max}}{n}}+\frac{\alpha m_1m_2\log m_{\max}}{n}\Big)\\
+C\|\bM\|\sigma_{\xi}\sqrt{\frac{\alpha m_1^2m_2\log m_{\max}}{n}}+C\|\bM\|_{\ell_\infty}\sigma_{\xi}\frac{\alpha m_1^{3/2}m_2\log m_{\max}}{n}.
\end{eqnarray*}
Denote this event by $\calE_6$.

We now bound $(n-1)^{-1}\Big\|\Big(n^{-1}m_1^2m_2^2\sum_{i=1}^nY_i\widetilde{Y}_i\bX_i\widetilde{\bX}_i\Big)-\bM\bM^{\top}\Big\|$. Write
\begin{eqnarray*}
&&\frac{m_1^2m_2^2}{n}\sum_{i=1}^nY_i\widetilde{Y}_i\bX_i\widetilde{\bX}_i^{\top}-\bM\bM^{\top}\\
&=&\frac{m_1^2m_2^2}{n}\sum_{i=1}^n\Big(\langle\bM,\bX_i\rangle\langle\bM,\widetilde{\bX}_i \rangle\bX_i\widetilde{\bX}_i^{\top}\Big)-\bM\bM^{\top}
+\frac{m_1^2m_2^2}{n}\sum_{i=1}^n\xi_i\langle\bM,\widetilde{\bX}_i\rangle\bX_i\widetilde{\bX}_i^{\top}\\
&&\qquad +\frac{m_1^2m_2^2}{n}\sum_{i=1}^n\widetilde{\xi}_i\langle\bM, \bX_i\rangle\bX_i\widetilde{\bX}_i^{\top}+\frac{m_1^2m_2^2}{n}\sum_{i=1}^{n}\xi_i\widetilde{\xi}_i\bX_i\widetilde{\bX}_i^{\top}.
\end{eqnarray*}
Clearly, all the four terms can be controlled by matrix Bernstein inequalities. To fix ideas, we consider only the last term. Indeed,
$$
\big\|\mathbb{E}m_1^4m_2^4\xi^2\widetilde{\xi}^2\bX\widetilde{\bX}^{\top}\widetilde{\bX}\bX^{\top}\big\|= m_1^3m_2^3\sigma_{\xi}^4,
$$
and based on properties of  Orlicz norms, see (\ref{eq:orlicznorm}),
$$
\big\|\|m_1^2m_2^2\xi\widetilde{\xi}\bX\widetilde{\bX}^{\top}\|\big\|_{\psi_1}\leq m_1^2m_2^2\|\xi\|_{\psi_2}\|\widetilde{\xi}\|_{\psi_2}\lesssim m_1^2m_2^2\sigma_{\xi}^2.
$$
By matrix Bernstein inequality (Lemma~\ref{lemma:matBern-unbound}), the following bound holds with probability at least $1-m_{\max}^{-\alpha}$ for $\alpha\geq 1$,
$$
\Big\|\frac{m_1^2m_2^2}{n}\sum_{i=1}^{n}\xi_i\widetilde{\xi}_i\bX_i\widetilde{\bX}_i^{\top}\Big\|\leq C\Big(m_1^{3/2}m_2^{3/2}\sigma_{\xi}^2\sqrt{\frac{\alpha\log m_{\max}}{n}}+m_1^2m_2^2\sigma_{\xi}^2\frac{\alpha\log^2m_{\max}}{n}\Big).
$$
We conclude that with probability at least $1-m_{\max}^{-\alpha}$,
\begin{eqnarray*}
\Big\|\frac{1}{n-1}\Big(\frac{m_1^2m_2^2}{n}\sum_{i=1}^nY_i\widetilde{Y}_i\bX_i\widetilde{\bX}_i^{\top}-\bM\bM^{\top}\Big)\Big\|\leq Cm_1^{3/2}m_2^{3/2}\big(\sigma_{\xi}^2+\|\bM\|_{\ell_\infty}^2\big)\sqrt{\frac{\alpha\log m_{\max}}{n^3}}\\
+Cm_1^2m_2^2\big(\sigma_{\xi}^2+\|\bM\|_{\ell_\infty}^2\big)\frac{\alpha\log^2m_{\max}}{n^2},
\end{eqnarray*}
where we used the fact $\sigma_{\xi}\|\bM\|_{\ell_\infty}\leq \frac{1}{2}\big(\|\bM\|_{\ell_\infty}^2+\sigma_{\xi}^2\big)$. Denote this event by $\calE_7$. 

To sum up, conditional on the event $\bigcap_{k=1}^7\calE_k$,
\begin{eqnarray*}
\big\|\widetilde{\bN}_1-\bN_1\big\|\leq C\|\bM\|_{\ell_\infty}\sigma_{\xi}\frac{\alpha^2(3\alpha+7)(m_1^3m_2^2m_{\max})^{1/2}\log^{3/2}m_{\max}\log^{1/2}n}{n}\\
+C\sigma_{\xi}^2\frac{\alpha^2(3\alpha+7)(m_1^3m_2^2m_{\max})^{1/2}\log^{3/2}m_{\max}\log^{1/2}n}{n}\\
+C\alpha(3\alpha+7)\|\bM\|_{\ell_\infty}^2\frac{(m_1^3m_2^2m_{\max})^{1/2}\log m_{\max}}{n}\\
+C\|\bM\|_{\ell_\infty}\|\bM\|\Big(\sqrt{\frac{\alpha m_1^2m_2\log m_{\max}}{n}}+\frac{\alpha m_1m_2\log m_{\max}}{n}\Big)\\
+C\|\bM\|\sigma_{\xi}\sqrt{\frac{\alpha m_1^2m_2\log m_{\max}}{n}}+C\|\bM\|_{\ell_\infty}\sigma_{\xi}\frac{\alpha m_1^{3/2}m_2\log m_{\max}}{n},
\end{eqnarray*}
assuming that  $n\geq C\alpha(\sqrt{m_1m_2}\log m_{\max}+m_1\log^2m_{\max})$. The bound can be further simplified to
\begin{eqnarray*}
\|\widetilde{\bN}_1-\bN_1\|&\leq& C\alpha^2(3\alpha+7)\sigma_{\xi}^2\frac{(m_1^3m_2^2m_{\max})^{1/2}\log^2m_{\max}\log n}{n}+C\|\bM\|\sigma_{\xi}\sqrt{\frac{\alpha m_1^2m_2\log m_{\max}}{n}}\\
&&+C\alpha(3\alpha+7)\|\bM\|_{\ell_\infty}^2\frac{(m_1^3m_2^2m_{\max})^{1/2}\log m_{\max}}{n}+C\|\bM\|_{\ell_\infty}\|\bM\|\sqrt{\frac{\alpha m_1^2m_2\log m_{\max}}{n}}.
\end{eqnarray*}
The proof is then completed by adjusting the constant $C$, replacing $m_1$, $m_2$ and $m_{\max}$ with $d_1$, $(d_2\ldots d_k)$ and $d_1 \vee (d_1\ldots d_k/d_1)$ respectively.

\subsection{Proof of Theorem~\ref{th:power}}
Theorem \ref{th:general} is a consequence of Theorem \ref{th:power}, and therefore, we present the proof of Theorem \ref{th:power} first. We divide the proof into two steps: we first bound the error of $U_j^{(\rm iter)}$ as an estimate of $U_j$, and then show how these bounds translate into bounds on the estimation error $\|\breve{\bT}-\bT\|_{\ell_2}$. Recall that 
$$
\bT=\bT\bigtimes_{j=1}^k U_j U_j^{\top},
$$
which will be used repeatedly. Denote by $P_{U_j}=U_jU_j^{\top}$ the projection onto the column space of $U_j$. 

\paragraph*{Upper bound of spectral estimation.} We begin with the first step by establishing upper bounds for $\big\|\big(U_j^{(\rm iter)}\big)\big(U_j^{(\rm iter)}\big)^{\top}-U_jU_j^{\top}\big\|$. Recall that $\{U_j^{(\rm iter)}: j=1,\ldots,k; {\rm iter}=0,1,\ldots,{\rm iter}_{\max}\}$ denote the sequence of spectral power iterations with initial value $\{U_1^{(0)},\ldots, U_k^{(0)}\}$. To this end,
define
$$
E_{\rm iter}:=\max\Big\{\|U_j^{(\rm iter)}(U_j^{(\rm iter)})^{\top}-U_jU_j^{\top}\|: 1\leq j\leq k\Big\},
$$
for ${\rm iter}=0,1,\ldots,{\rm iter}_{\max}$. Note that $E_0\leq \frac{1}{2}$.

Recall that $U_j^{\rm (iter+1)}$ are the top $r_j$ left singular vectors of
\begin{align*}
\calM_j\big(\hat{\bT}^{\rm init}\times_{j'<j} U_{j'}^{(\rm iter+1)}\times_{j'>j} U_{j'}^{\rm (iter)}\big)=\calM_j\big(&\bT\times_{j'<j}U_{j'}^{(\rm iter+1)}\times_{j'>j} U_{j'}^{\rm (iter)}\big)\\
+&\calM_j\big((\hat{\bT}^{\rm init}-\bT)\times_{j'<j} U_{j'}^{(\rm iter+1)}\times_{j'>j} U_{j'}^{\rm (iter)}\big),
\end{align*}
where $U_j$ are the top $r_j$ left singular vectors of $\calM_j\big(\bT\times_{j'<j} U_{j'}^{(\rm iter+1)}\times_{j'>j} U_{j'}^{\rm (iter)}\big)$.  Moreover,
\begin{eqnarray*}
&&\sigma_{r_j}\Big(\calM_j\big(\bT\times_{j'<j} U_{j'}^{(\rm iter+1)}\times_{j'>j} U_{j'}^{\rm (iter)}\big)\Big)\\
&=&\sigma_{r_j}\Big(\calM_j(\bT)(\bigotimes_{j'\neq j} P_{U_{j'}})\big(\bigotimes_{j'<j} U_{j'}^{\rm (iter+1)}
\bigotimes_{j'>j} U_{j'}^{\rm (iter)}\big)\Big)\\
&=&\sigma_{r_j}\Big(\calM_j(\bT)(\bigotimes_{j'\neq j} U_{j'})\big(\bigotimes_{j'\neq j}U_{j'}\big)^{\top}\big(\bigotimes_{j'<j} U_{j'}^{\rm (iter+1)}\bigotimes_{j'>j} U_{j'}^{\rm (iter)}\big)\Big)\\
&\geq& \sigma_{r_j}\Big(\calM_j(\bT)\big(\bigotimes_{j'\neq j}U_{j'}\big)\Big)\sigma_{\min}\Big(\big(\bigotimes_{j'\neq j}U_{j'}\big)^{\top}\big(\bigotimes_{j'<j} U_{j'}^{\rm (iter+1)}
\bigotimes_{j'>j} U_{j'}^{\rm (iter)}\big)\Big)\\
&=&\sigma_{r_j}\big(\calM_j(\bT)\big)\sigma_{\min}\Big(\bigotimes_{j'<j}(U_{j'}^{\top}U_{j'}^{\rm(iter+1)})\bigotimes_{j'>j}(U_{j'}^{\top}U_{j'}^{\rm(iter)})\Big)\\
&\geq& \Lambda_{\min}(\bT)\prod_{j'<j}\sigma_{\min}\big(U_{j'}^{\top}U_{j'}^{\rm (iter+1)}\big)\prod_{j'>j}\sigma_{\min}\big(U_{j'}^{\top}U_{j'}^{\rm(iter)}\big)\\
&\geq& \Lambda_{\min}(\bT)\prod_{j'<j}\sqrt{1-\|U_{j'}^{\rm(iter+1)}(U_{j'}^{\rm (iter+1)})^{\top}-U_{j'}U_{j'}^{\top}\|}\prod_{j'>j}\sqrt{1-\|U_{j'}^{\rm(iter)}(U_{j'}^{\rm (iter)})^{\top}-U_{j'}U_{j'}^{\top}\|}\\
&\geq& \Lambda_{\min}(\bT)(1-E_{\rm iter+1})^{(j-1)/2}(1-E_{\rm iter})^{(k-j)/2},
\end{eqnarray*}
where we used the fact
\begin{eqnarray*}
\sigma_{\min}(U_{j'}^{\top}U_{j'}^{\rm (iter)})=\sqrt{\sigma_{\min}(U_{j'}^{\top}U_{j'}^{\rm(iter)}(U_{j'}^{\rm(iter)})^{\top}U_{j'})}\geq \sqrt{1-\|U_{j'}^{\rm(iter)}(U_{j'}^{\rm(iter)})^{\top}-U_{j'}U_{j'}^{\top}\|}\\
\geq \sqrt{1-E_{\rm iter}}.
\end{eqnarray*}
Next, we control $\big\|\calM_j\big((\hat\bT^{\rm init}-\bT)\times_{j'<j}U_{j'}^{\rm(iter+1)}\times_{j'>j}U_{j'}^{\rm(iter)}\big)\big\|$. We write
\begin{align*}
\calM_j\big((\hat\bT^{\rm init}-\bT)\times_{j'<j}&U_{j'}^{\rm(iter+1)}\times_{j'>j}U_{j'}^{\rm(iter)}\big)\\
=&\calM_j\big((\hat\bT^{\rm init}-\bT)\times_{j'<j}(P_{U_{j'}}+P_{U_{j'}}^{\perp})U_{j'}^{\rm(iter+1)}\times_{j'>j}(P_{U_{j'}}+P_{U_{j'}}^{\perp})U_{j'}^{\rm(iter)}\big)\\
=:&\calM_j\big((\hat\bT^{\rm init}-\bT)\times_{j'<j}(P_{U_{j'}}U_{j'}^{\rm(iter+1)})\times_{j'>j}(P_{U_{j'}}U_{j'}^{\rm(iter)})\big)+\hat{\bR}.
\end{align*}
We shall make use of the following lemma whose proof is relegated to the Appendix.
\begin{lemma}\label{lemma:hatR}
Let $\bT\in\Theta(r_1,\ldots,r_k)$ and $\xi$ be subgaussian in that there exists a $\sigma_{\xi}>0$ such that for all $s\in\RR$,
$$
\EE(\exp\{s\xi\})\leq \exp\{s^2\sigma_{\xi}^2/2\}.
$$
There exist absolute constants $C_1,C_2,C_3>0$ such that if
$$
n\geq C_1k\alpha\big(\beta(\bT)\kappa(\bT)\big)^{2(k-1)}d_{\max}\log (d_{\max}),
$$
for any fixed $\alpha>1$, then the following bounds hold with probability at least $1-d^{-\alpha}_{\max}$,
\begin{align*}
\big\|\calM_j\big((\hat\bT^{\rm init}-\bT)\times_{j'<j}&(P_{U_{j'}}U_{j'}^{\rm(iter+1)})\times_{j'>j}(P_{U_{j'}}U_{j'}^{\rm(iter)})\big)\big\|\\
\leq& C_2\big(\|\bT\|_{\ell_\infty}\vee \sigma_{\xi}\big)\Big(d_j\vee \frac{r_1\ldots r_k}{r_j}\Big)^{1/2}\sqrt{\frac{\alpha k d_1\ldots d_{k}\log (d_{\max})}{n}},
\end{align*}
and
\begin{align*}
\|\hat{\bR}\|
\leq (E_{\rm iter}\vee E_{\rm iter+1})C_{3}\alpha k^{k+4}&r_{\max}(\bT)^{(k-2)/2}\big(\|\bT\|_{\ell_\infty}\vee \sigma_{\xi}\big)\\
\times&\max\bigg\{\sqrt{\frac{kd_{\max}d_1\ldots d_k}{n}},\frac{kd_1\ldots d_k}{n}\bigg\}\log^{k+2}d_{\max}.
\end{align*}
\end{lemma}
By Lemma~\ref{lemma:hatR}, we conclude that
\begin{eqnarray*}
&&\big\|\calM_j\big((\hat\bT^{\rm init}-\bT)\times_{j'<j}U_{j'}^{\rm(iter+1)}\times_{j'>j}U_{j'}^{\rm(iter)}\big)\big\|\\
&\leq& C_2\big(\|\bT\|_{\ell_\infty}\vee \sigma_{\xi}\big)\Big(d_j\vee \frac{r_1\ldots r_k}{r_j}\Big)^{1/2}\sqrt{\frac{\alpha k d_1\ldots d_{k}\log (d_{\max})}{n}}
+(E_{\rm iter}\vee E_{\rm iter+1})C_{3}\alpha\\
&&\times k^{k+4}r_{\max}(\bT)^{(k-2)/2}\big(\|\bT\|_{\ell_\infty}\vee \sigma_{\xi}\big)\max\bigg\{\sqrt{\frac{kd_{\max}d_1\ldots d_k}{n}},\frac{kd_1\ldots d_k}{n}\bigg\}\log^{k+2}d_{\max},
\end{eqnarray*}
with probability at least $1-d^{-\alpha}_{\max}$. 

Applying Wedin's $\sin\Theta$ theorem \citep{wedin1972perturbation}, we conclude that with probability at least $1-d^{-\alpha}_{\max}$,
\begin{eqnarray*}
&&\big\|U_j^{\rm (iter+1)}(U_j^{\rm (iter+1)})^\top-U_jU_j^\top\big\|\\
&\leq& \frac{2\big\|\calM_j\big((\hat\bT^{\rm init}-\bT)\times_{j'<j}(P_{U_{j'}}U_{j'}^{\rm(iter+1)})\times_{j'>j}(P_{U_{j'}}U_{j'}^{\rm(iter)})\big)\big\|}{\Lambda_{\min}(\bT)(1-E_{\rm iter+1})^{(j-1)/2}(1-E_{\rm iter})^{(k-j)/2}}
+(E_{\rm iter}\vee E_{\rm iter+1})\\
&&\times\frac{C_{3}\alpha k^{k+4}r_{\max}(\bT)^{(k-2)/2}\big(\|\bT\|_{\ell_\infty}\vee \sigma_{\xi}\big)}{\Lambda_{\min}(\bT)(1-E_{\rm iter+1})^{(j-1)/2}(1-E_{\rm iter})^{(k-j)/2}}\max\bigg\{\sqrt{\frac{kd_{\max}d_1\ldots d_k}{n}},\frac{kd_1\ldots d_k}{n}\bigg\}\log^{k+2}d_{\max}.
\end{eqnarray*}
It is easy to check that if $\max\big\{E_{\rm iter}, E_{\rm iter+1}\big\}\leq \frac{1}{2}$ and 
\begin{eqnarray*}
n\geq C_4\max\bigg\{\alpha^22^kk^{2(k+4)}r_{\max}(\bT)^{k-2}\Lambda_{\min}^{-2}(\bT)\big(\|\bT\|_{\ell_\infty}\vee \sigma_{\xi}\big)^2kd_{\max}(d_1\ldots d_k)\log^{2(k+2)}d_{\max},\nonumber\\
\alpha 2^{k/2}k^{k+5}r_{\max}(\bT)^{(k-2)/2}\Lambda_{\min}^{-1}(\bT)\big(\|\bT\|_{\ell_\infty}\vee \sigma_{\xi}\big)d_1\ldots d_k\log^{k+2}d_{\max}\bigg\}
\end{eqnarray*}
for some large enough constant $C_4>0$, then 
\begin{eqnarray*}
&&\big\|U_j^{\rm (iter+1)}(U_j^{\rm (iter+1)})^\top-U_jU_j^\top\big\|\\
&\leq& \frac{1}{2}\big(E_{\rm iter}\vee E_{\rm iter+1}\big)
+C_22^{k/2}\frac{(\|\bT\|_{\ell_\infty}\vee \sigma_{\xi})}{\Lambda_{\min}(\bT)}\Big(d_j\vee \frac{r_1\ldots r_k}{r_j}\Big)^{1/2}\sqrt{\frac{\alpha k d_1\ldots d_{k}\log (d_{\max})}{n}}.
\end{eqnarray*}
Note that this bound applies to all $j=1,\ldots, k$. Therefore,
$$
E_{\rm iter+1}\leq \frac{1}{2}\big(E_{\rm iter}\vee E_{\rm iter+1}\big)+C_22^{k/2}\frac{(\|\bT\|_{\ell_\infty}\vee \sigma_{\xi})}{\Lambda_{\min}(\bT)}\Big(d_{\max}\vee \frac{r_1\ldots r_k}{\min_{1\leq j\leq k}r_j}\Big)^{1/2}\sqrt{\frac{\alpha k d_1\ldots d_{k}\log (d_{\max})}{n}}.
$$
It is easy to show that after ${\rm iter}_{\max}=Ck\log d_{\max}$, under the lower bound on $n$, we get
\begin{align*}
E_{\rm iter_{\max}}\leq C_22^{k/2}\frac{(\|\bT\|_{\ell_\infty}\vee \sigma_{\xi})}{\Lambda_{\min}(\bT)}\Big(d_{\max}\vee \frac{r_1\ldots r_k}{\min_{1\leq j\leq k}r_j}\Big)^{1/2}\sqrt{\frac{\alpha k d_1\ldots d_{k}\log (d_{\max})}{n}}.
\end{align*}

\paragraph*{Upper bound of $\|\breve\bT-\bT\|_{\ell_2}$.} We are now in position to prove the upper bound of $\|\breve{\bT}-\bT\|_{\ell_2}$ where
$$
\breve{\bT}=\hat{\bT}^{\rm init}\times_{j=1}^k U_j^{(\rm iter)}\big(U_j^{\rm(iter)}\big)^{\top},
$$
for ${\rm iter}\geq {\rm iter}_{\rm max}$.
Write
\begin{eqnarray*}
\|\breve{\bT}-\bT\|_{\ell_2}&=&\big\|\hat{\bT}^{\rm init}\times_{j=1}^kU_j^{(\rm iter)}\big(U_j^{\rm(iter)}\big)^{\top}-\bT\times_{j=1}^k U_j(U_j)^{\top}\big\|_{\ell_2}\\
&\leq&\big\|\big(\hat{\bT}^{\rm init}-\bT\big)\times_{j=1}^kU_j^{(\rm iter)}\big(U_j^{\rm(iter)}\big)^{\top}\big\|_{\ell_2}\\
&&+\sum_{j=1}^{k}\big\|\bT\times_{j'<j}U_{j'}^{\rm(iter)}\big(U_{j'}^{\rm(iter)}\big)^{\top}\times_j\big(U_j^{\rm(iter)}(U_j^{\rm(iter)})^{\top}-U_jU_j^{\top}\big)\times_{j'>j}U_{j'}(U_{j'})^{\top}\big\|_{\ell_2}.
\end{eqnarray*}
We apply the following lemma whose proof is relegated to the Appendix.
\begin{lemma}\label{lemma:TUiter-U}
There exists a constant $C_1,C_2>0$ depending on $k$ only such that 
for all $1\leq j\leq k$,  if
\begin{align*}
n\geq C_1(k)\max\bigg\{\alpha^2 r_{\max}(\bT)^{k-2}\Lambda_{\min}^{-2}(\bT)\big(\|\bT\|_{\ell_\infty}\vee \sigma_{\xi}\big)^2kd_{\max}(d_1\ldots d_k)\log^{2(k+2)}d_{\max},\nonumber\\
\alpha r_{\max}(\bT)^{(k-2)/2}\Lambda_{\min}^{-1}(\bT)\big(\|\bT\|_{\ell_\infty}\vee \sigma_{\xi}\big)d_1\ldots d_k\log^{k+2}d_{\max},\\
\alpha \kappa(\bT)^2\Lambda_{\min}^{-2}(\bT)\big(\|\bT\|_{\ell_\infty}\vee \sigma_{\xi}\big)^2\big(d_{\max}\vee r_{\max}(\bT)^{k-1}\big)d_1\ldots d_k\log d_{\max}\bigg\}
\end{align*}
the following bound holds with probability at least $1-d_{\max}^{-\alpha}$ for any $\alpha\geq 1$,
\begin{align*}
\big\|\bT\times_{j'<j}U_{j'}^{\rm(iter)}\big(U_{j'}^{\rm(iter)}\big)^{\top}\times_j\big(U_j^{\rm(iter)}(U_j^{\rm(iter)})^{\top}-U_jU_j^{\top}\big)\times_{j'>j}U_{j'}(U_{j'})^{\top}\big\|_{\ell_2}\\
\leq C_2(k)\big(\|\bT\|_{\infty}\vee \sigma_{\xi}\big)\Big(d_{\max}r_{\max}(\bT)\vee \frac{r_1\ldots r_k}{\min_{1\leq j\leq k}r_j/r_{\max}(\bT)}\Big)^{1/2}\sqrt{\frac{\alpha k d_1\ldots d_k\log (d_{\max})}{n}}.
\end{align*}
\end{lemma}
By Lemma~\ref{lemma:TUiter-U}, we obtain
\begin{eqnarray*}
&&\|\breve{\bT}-\bT\|_{\ell_2}\\
&\leq& \big\|\big(\hat{\bT}^{\rm init}-\bT\big)\times_{j=1}^kU_j^{(\rm iter)}\big(U_j^{\rm(iter)}\big)^{\top}\big\|_{\ell_2}\\
&&+C_2(k)\big(\|\bT\|_{\ell_\infty}\vee \sigma_{\xi}\big)\Big(d_{\max}r_{\max}(\bT)\vee \frac{r_1\ldots r_k}{\min_{1\leq j\leq k}r_j/r_{\max}(\bT)}\Big)^{1/2}\sqrt{\frac{\alpha k d_1\ldots d_k\log (d_{\max})}{n}}.
\end{eqnarray*}
It remains to bound the first term on the rightmost hand side. Note that
\begin{eqnarray*}
&&\big\|\big(\hat{\bT}^{\rm init}-\bT\big)\times_{j=1}^kU_j^{(\rm iter)}\big(U_j^{\rm(iter)}\big)^{\top}\big\|_{\ell_2}\\
&\leq& \big\|\big(\hat{\bT}^{\rm init}-\bT\big)\times_{j=1}^kU_j\big(U_j\big)^{\top}\big\|_{\ell_2}\\
&&+\sum_{j=1}^k\big\|(\hat{\bT}^{\rm init}-\bT)\times_{j'<j}U_{j'}U_{j'}^{\top}\times_j\big(U_j^{(\rm iter)}\big(U_j^{\rm(iter)}\big)^{\top}-U_jU_j^{\top}\big)\times_{j'>j}U_{j'}^{(\rm iter)}\big(U_{j'}^{\rm(iter)}\big)^{\top}\big\|_{\ell_2}.
\end{eqnarray*}
In the light of Lemma \ref{lemma:hatR},
\begin{eqnarray*}
&&\big\|\big(\hat{\bT}^{\rm init}-\bT\big)\times_{j=1}^kU_j\big(U_j\big)^{\top}\big\|_{\ell_2}=\Big\|\calM_1\Big(\big(\hat{\bT}^{\rm init}-\bT\big)\times_{j=1}^kU_j\big(U_j\big)^{\top}\Big)\Big\|_{\ell_2}\\
&\leq& \sqrt{r_1}\Big\|\calM_1\Big(\big(\hat{\bT}^{\rm init}-\bT\big)\times_{j=1}^kU_j\big(U_j\big)^{\top}\Big)\Big\|\\
&\leq& C_2\big(\|\bT\|_{\ell_\infty}\vee \sigma_{\xi}\big)\big(d_1r_1\vee r_1\ldots r_k\big)^{1/2}\sqrt{\frac{\alpha kd_1\ldots d_k\log (d_{\max})}{n}},
\end{eqnarray*}
with probability at least $1-d_{\max}^{-\alpha}$. 
Similarly as Lemma~\ref{lemma:hatR}, we obtain for each $j=1,\ldots,k$ that
 \begin{eqnarray*}
&&\big\|(\hat{\bT}^{\rm init}-\bT)\times_{j'<j}U_{j'}U_{j'}^{\top}\times_j\big(U_j^{(\rm iter)}\big(U_j^{\rm(iter)}\big)^{\top}-U_jU_j^{\top}\big)\times_{j'>j}U_{j'}^{(\rm iter)}\big(U_{j'}^{\rm(iter)}\big)^{\top}\big\|_{\ell_2}\\
&\leq& \Big\|\calM_j\Big((\hat{\bT}^{\rm init}-\bT)\times_{j'<j}U_{j'}U_{j'}^{\top}\times_{j'>j}U_{j'}^{\rm(iter)}\big(U_{j'}^{\rm(iter)}\big)^{\top}\Big)\Big\|\big\|U_j^{(\rm iter)}\big(U_j^{\rm(iter)}\big)^{\top}-U_jU_j^{\top}\big\|_{\ell_2}\\
&\leq& C_1\alpha k^{k+4}(\|\bT\|_{\ell_\infty}\vee \sigma_{\xi})r_{\max}(\bT)^{(k-2)/2}\max\bigg\{\sqrt{\frac{kd_{\max}d_1\ldots d_k}{n}}, \frac{kd_1\ldots d_k}{n}\bigg\}\log^{k+2}d_{\max}\times\\
&&\qquad 2^{k/2}\frac{(\|\bT\|_{\ell_\infty}\vee \sigma_{\xi})}{\Lambda_{\min}(\bT)}\Big(r_{\max}(\bT)d_{\max}\vee \frac{r_1\ldots r_k}{\min_{1\leq j\leq k}r_j/r_{\max}(\bT)}\Big)^{1/2}\sqrt{\frac{\alpha kd_1\ldots d_k\log (d_{\max})}{n}}.
 \end{eqnarray*}
 Clearly, if the sample size
 \begin{eqnarray*}
 n\geq C_1\max\bigg\{\alpha 2^{k/2}k^{k+5}\Lambda_{\min}^{-1}(\bT)(\|\bT\|_{\ell_\infty}\vee \sigma_{\xi})r_{\max}(\bT)^{(k-2)/2}d_1\ldots d_k\log^{k+2}d_{\max},\\
 \alpha^22^kk^{2k+9}\Lambda_{\min}^{-2}(\bT)(\|\bT\|_{\ell_\infty}\vee \sigma_{\xi})^2r_{\max}(\bT)^{k-2}d_{\max}(d_1\ldots d_k)\log^{2k+4}d_{\max}\bigg\},
 \end{eqnarray*}
then
 \begin{eqnarray*}
&&\big\|(\hat{\bT}^{\rm init}-\bT)\times_{j'<j}U_{j'}U_{j'}^{\top}\times_j\big(U_j^{(\rm iter)}\big(U_j^{\rm(iter)}\big)^{\top}-U_jU_j^{\top}\big)\times_{j'>j}U_{j'}^{(\rm iter)}\big(U_{j'}^{\rm(iter)}\big)^{\top}\big\|_{\ell_2}\\
&\leq& C_2(\|\bT\|_{\ell_\infty}\vee \sigma_{\xi})\Big(d_{\max}r_{\max}(\bT)\vee \frac{r_1\ldots r_k}{\min_{1\leq j\leq k}r_j/r_{\max}(\bT)}\Big)^{1/2}\sqrt{\frac{\alpha  kd_1\ldots d_k\log (d_{\max})}{n}}.
 \end{eqnarray*}
Therefore, we conclude that with probability at least $1-d_{\max}^{-\alpha}$,
$$
\frac{\|\breve{\bT}-\bT\|_{\ell_2}}{(d_1\ldots d_k)^{1/2}}\leq C_2(k)\big(\|\bT\|_{\ell_\infty}\vee \sigma_{\xi}\big)\Big(d_{\max}r_{\max}(\bT)\vee \frac{r_1\ldots r_k}{\min_{1\leq j\leq k}r_j/r_{\max}(\bT)}\Big)^{1/2}\sqrt{\frac{\alpha k\log (d_{\max})}{n}}.
$$
Bounds under general $\ell_p$ norm follows immediately from the fact that
$$
\frac{\|\breve{\bT}-\bT\|_{\ell_p}}{(d_1\ldots d_k)^{1/p}}\leq\frac{\|\breve{\bT}-\bT\|_{\ell_2}}{(d_1\ldots d_k)^{1/2}},
$$
by Cauchy-Schwartz inequality.

\subsection{Proof of Theorem~\ref{th:general}}
The proof of Theorem~\ref{th:general} is based on Theorem~\ref{th:power} and Corollary~\ref{co:main} where $\Lambda_{\min}(\bT)$ is assumed to satisfy
\begin{align}
\Lambda_{\min}(\bT)\geq c\alpha^{3/2}\big(\|\bT\|_{\ell_\infty}\vee \sigma_{\xi}\big)\log^{k+2}d_{\max}\bigg(\kappa(\bT)r_{\max}(\bT)^{(k-2)/2}\sqrt{\frac{d_{\max}d_1\ldots d_k}{n}}\nonumber\\
+\frac{(d_1\ldots d_k)^{3/4}}{n^{1/2}}+r_{\max}(\bT)^{(k-2)/2}\frac{d_1\ldots d_k}{n}\bigg).\label{eq:Lambdamin}
\end{align}
We now prove Theorem~\ref{th:general} in two separate cases depending on whether or not \eqref{eq:Lambdamin} holds.

{\it Case 1:} if the lower bound (\ref{eq:Lambdamin}) on $\Lambda_{\min}$ holds, then $U_j^{(0)}$ has exactly rank $r_j$ with overwhelming probability. We can apply Corollary~\ref{co:main} to immediately get, with probability at least $1-d_{\max}^{-\alpha}$,
$$
\frac{\|\hat{\bT}-\bT\|_{\ell_p}}{(d_1\ldots d_k)^{1/p}}\leq C_3\big(\sigma_{\xi}\vee \|\bT\|_{\ell_\infty}\big)\sqrt{\frac{\alpha(r_{\max}(\bT)d_{\max}\vee r_{\max}(\bT)^k)\log d_{\max}}{n}}
$$
for all $1\leq p\leq 2$.

{\it Case 2:} if the lower bound (\ref{eq:Lambdamin}) does not hold, meaning that
\begin{align*}
\Lambda_{\min}(\bT)\leq & c\alpha^{3/2}\big(\|\bT\|_{\ell_\infty}\vee \sigma_{\xi}\big)\log^{k+2}d_{\max}\times\nonumber\\
&\times\bigg(\kappa(\bT)r_{\max}(\bT)^{(k-2)/2}\sqrt{\frac{d_{\max}d_1\ldots d_k}{n}}+\frac{(d_1\ldots d_k)^{3/4}}{n^{1/2}}+r_{\max}(\bT)^{(k-2)/2}\frac{d_1\ldots d_k}{n}\bigg).
\end{align*}
Then, we can write
\begin{align*}
\|\hat\bT-\bT\|_{\ell_2}=&\|\hat\bT^{\rm init}\bigtimes_{j=1}^k P_{\hat U_j}-\bT\|_{\ell_2}\\
\leq& \big\|\big(\hat\bT^{\rm init}-\bT\big)\bigtimes_{j=1}^k P_{\hat U_j}\big\|_{\ell_2}+\big\|\bT\bigtimes_{j=1}^k P_{\hat U_j}-\bT\big\|_{\ell_2}\\
\leq& \big\|\big(\hat\bT^{\rm init}-\bT\big)\bigtimes_{j=1}^k P_{\hat U_j}\big\|_{\ell_2}+2\kappa(\bT)r_{\max}(\bT)^{1/2}\Lambda_{\min}(\bT).
\end{align*}
Observe that $\big(\hat\bT^{\rm init}-\bT\big)\bigtimes_{j=1}^k P_{\hat U_j}$ has multilinear ranks at most $(r_{\max}(\bT),\ldots,r_{\max}(\bT))$. By Theorem~\ref{th:init},
\begin{align*}
\|\hat\bT-\bT\|_{\ell_2}\leq& r_{\max}(\bT)^{(k-1)/2}\big\|\hat\bT^{\rm init}-\bT\big\|+2\kappa(\bT)r_{\max}(\bT)^{1/2}\Lambda_{\min}(\bT)\\
\leq& c\alpha^{3/2}\kappa(\bT)\big(\|\bT\|_{\ell_\infty}\vee \sigma_{\xi}\big)\log^{k+2}d_{\max}\bigg(\kappa(\bT)r_{\max}(\bT)^{(k-1)/2}\sqrt{\frac{d_{\max}d_1\ldots d_k}{n}}\\
&\hskip 100pt +r_{\max}(\bT)^{1/2}\frac{(d_1\ldots d_k)^{3/4}}{n^{1/2}}+r_{\max}(\bT)^{(k-1)/2}\frac{d_1\ldots d_k}{n}\bigg),
\end{align*}
which completes the proof.

\subsection{Proof of Theorem~\ref{thm:minimax}}
Assume that, without loss of generality, $d_1=d_{\max}=\max\{d_1,\ldots,d_k\}$. Define the set of matrices
$$
\calA:=\bigg\{A\in\mathbb{R}^{d_1\times r_1}: A(i,j)\in\Big\{0, \gamma(M\wedge \sigma_{\xi})\sqrt{\frac{r_0d_{\max}}{n}}\Big\}, i=1,\ldots,d_1; j=1,\ldots,r_1\bigg\}
$$
for some constant $\gamma>0$. By Varshamov-Gilbert bound (see \cite{koltchinskii2011nuclear}), there exists a subset $\widetilde{\calA}\subset \calA$ with ${\rm Card}(\widetilde{\calA})\geq 2^{r_0d_{\max}/8}$ such that for any $A_1\neq A_2\in \widetilde{\calA}$,
$$
\|A_1-A_2\|_{\ell_p}\geq \Big(\frac{r_0d_{\max}}{8}\Big)^{1/p}\gamma(M\wedge \sigma_{\xi})\sqrt{\frac{r_0d_{\max}}{n}}
$$
and the sparsity
$$
\|A_1\|_{\ell_0}=\|A_2\|_{\ell_0}=\frac{r_{0}d_{\max}}{2}.
$$
We then construct a subset of block low rank tensors
$$
\calT(r_0,2)=\bigg\{\bT(A)=\Big(A|\ldots|A|{\bf 0}\Big)\otimes {1}_{d_3}\otimes\ldots\otimes {1}_{d_k}\in\mathbb{R}^{d_1\times\ldots\times d_k}: A\in\widetilde{\calA}\bigg\}
$$
where ${\bf 0}$ represents a $d_1\times (d_2-r_0\floor{d_2/r_0})$ zero matrix and $1_{d_3}=(1,1,\ldots,1)^{\top}\in\mathbb{R}^{d_3}$ is an all-one vector. Clearly, by the construction, we have $\max_{1\leq j\leq k}r_j(\bT)\leq r_0$ for all $\bT\in\calT(r_0,2)$ and $\|\bT\|_{\ell_\infty}/\|\bT\|_{\ell_2}\leq 2(d_1\ldots d_k)^{-1/2}$. Moreover, for any $\bT_1\neq \bT_2\in \calT(r_0,2)$, we have
$$
\frac{1}{(d_1\ldots d_k)^{1/p}}\|\bT_1-\bT_2\|_{\ell_p}\geq \frac{1}{10}\gamma\big(M\wedge \sigma_{\xi}\big)\sqrt{\frac{r_0d_{\max}}{n}}.
$$
Denote by $d_{\rm KL}(\bT_1,\bT_2)$ the Kullback-Leibler divergence between $\PP_{\bT_1}$ and $\PP_{\bT_2}$. By a standard argument, we get
\begin{eqnarray*}
d_{\rm KL}(\PP_{\bT_1},\PP_{\bT_2})&=&\EE_{\PP_{\bT_1}}\log\frac{\PP_{\bT_1}}{\PP_{\bT_2}}\Big(\omega_1,Y_1,\omega_2,Y_2,\ldots,\omega_n,Y_n\Big)\\
&=&\EE_{\PP_{\bT_1}}\sum_{i=1}^n\Big(-\frac{\big(Y_i-T_1(\omega_i)\big)^2}{2\sigma_{\xi}^2}+\frac{\big(Y_i-T_2(\omega_i)\big)^2}{2\sigma_{\xi}^2}\Big)\\
&=&\EE\sum_{i=1}^n\frac{\big(T_1(\omega_i)-T_2(\omega)\big)^2}{2\sigma_{\xi}^2}\\
&=&\frac{n}{2d_1\ldots d_k\sigma_{\xi}^2}\|\bT_1-\bT_2\|_{\ell_2}^2.
\end{eqnarray*}
It follows that if $\bT_1,\bT_2\in \calT(r_0,2)$, then
\begin{eqnarray*}
d_{\rm KL}\big(\PP_{\bT_1}, \PP_{\bT_2}\big)&\leq& \frac{n}{\sigma_{\xi}^2d_1\ldots d_k}\big(\|\bT_1\|_{\ell_2}^2+\|\bT_2\|_{\ell_2}^2\big)\\
&\leq& \frac{2n}{\sigma_{\xi}^2}\gamma^2(M\wedge \sigma_{\xi})^2\frac{r_0d_{\max}}{n}\\
&\leq& \log{\rm Card}\big(\calT(r_0, 2)\big),
\end{eqnarray*}
where the last inequality holds by taking the constant $\gamma$ small enough. By Fano's lemma \citep{tsybakov2008intro},
$$
\inf_{\tilde{\bT}}\ \sup_{\bT\in\calT(r_0,2)}\ \PP_{\bT}\Big(\frac{1}{(d_1\ldots d_k)^{p}}\|\tilde{\bT}-\bT\|_{\ell_p}\geq C_1\big(M\wedge \sigma_{\xi}\big)\sqrt{\frac{r_0d_{\max}}{n}}\Big)\geq C_2
$$
for some absolute constants $C_1,C_2>0$.

On the other hand, we can consider another set of tensors
$$
\calB:=\Big\{\bB\in\RR^{r_0\times \ldots\times r_0}: B(i_{j_1},\ldots,i_{j_k})\in\Big\{0, \gamma\big(M\wedge \sigma_{\xi}\big)\sqrt{\frac{r_0^k}{n}}\Big\}, 1\leq i_{j_1},\ldots,i_{j_k}\leq r_0\Big\}
$$
for some constant $\gamma>0$. By Varshamov-Gilbert bound, there exists a subset $\widetilde{\calB}\subset \calB$ with ${\rm Card}(\widetilde{\calB})\geq 2^{r_0^k/8}$ such that for any $\bB_1\neq \bB_2\in\widetilde{\calB}$,
$$
\|\bB_1-\bB_2\|_{\ell_p}\geq \Big(\frac{r_0^k}{8}\Big)^{1/p}\gamma\big(M\wedge \sigma_{\xi}\big)\sqrt{\frac{r_0^k}{n}}
$$
with the sparsity
$$
\|\bB_1\|_{\ell_0}=\|\bB_2\|_{\ell_0}=\frac{r_0^k}{2}.
$$
Following the same analysis, we obtain
$$
\inf_{\tilde{\bT}}\ \sup_{\bT\in\calT(r_0,2)}\ \PP_{\bT}\bigg(\frac{1}{(d_1\ldots d_k)^{p}}\|\tilde{\bT}-\bT\|_{\ell_p}\geq C_1\big(\|\bT\|_{\ell_\infty}\wedge \sigma_{\xi}\big)\sqrt{\frac{r_0^k}{n}}\bigg)\geq C_2
$$
which concludes the proof in view of the fact $\calT(r_0,2)\subset \Theta(r_0,2)\subset \Theta(r_0,\beta_0)$ for any $\beta_0\geq 2$.

\bibliographystyle{plainnat}
\bibliography{refer}

\appendix

\section{Matrix Bernstein Inequalities}
In the proof, we made repeated use of several versions of matrix Bernstein inequalities which we include here for completeness.

\begin{lemma}\label{lemma:matBern-bounded}\citep{tropp2012user}
Let $\bX_1,\ldots,\bX_n\in\mathbb{R}^{m_1\times m_2}$ be random matrices with zero mean. Suppose that for some $U>0$, $\max_{1\leq i\leq n}\|\bX_i\|\leq U$ a.s. Let 
$$
\sigma^2:=\max\Big\{\Big\|\sum_{i=1}^n\mathbb{E}\bX_i\bX_i^{\top}\Big\|, \Big\|\sum_{i=1}^n\mathbb{E}\bX_i^{\top}\bX_i\Big\|\Big\}.
$$ Then, for all $t\geq 0$, the following bound holds with probability at least $1-e^{-t}$,
$$
\Big\|\frac{\bX_1+\ldots+\bX_n}{n}\Big\|\leq 2\max\bigg\{\sigma\frac{\sqrt{t+\log(m_1+m_2)}}{n}, U\frac{t+\log(m_1+m_2)}{n}\bigg\}.
$$
\end{lemma}
Lemma~\ref{lemma:matBern-bounded} applies to bounded random variables. Another version of the matrix Bernstein inequality deals with the case when $\|\bX\|$ is unbounded but has an exponential tail. 

\begin{lemma}\label{lemma:matBern-unbound}\citep{minsker2011some}
Let $\bX_1,\ldots,\bX_n\in\mathbb{R}^{m_1\times m_2}$ be random matrices with zero mean. Suppose that $\max_{1\leq i\leq n}\big\|\|\bX_i\|\big\|_{\psi_\alpha}\leq U^{(\alpha)}<\infty$ for some $\alpha\geq 1$. Then there exists a universal constant $C>0$ such that for all $t>0$, the following bound holds with probability at least $1-e^{-t}$,
$$
\Big\|\frac{\bX_1+\ldots+\bX_n}{n}\Big\|\leq C\max\bigg\{\sigma\frac{\sqrt{t+\log(m_1+m_2)}}{n}, U^{(\alpha)}\Big(\log\frac{\sqrt{n}U^{(\alpha)}}{\sigma}\Big)\frac{t+\log(m_1+m_2)}{n}\bigg\}.
$$
\end{lemma}

\section{Proof of Lemma~\ref{lemma:Dvjk}}
The proof follows the same argument as that for Lemma 12 of \cite{yuan2015tensor}.
Denote the aspect ratio for a block $A_1\times\ldots A_k\subset [d_1]\times\ldots\times [d_k]$,
$$
h(A_1\times \ldots\times A_k)=\min\Big\{\nu: |A_j|^2\leq \nu\prod_{j=1}^k|A_j|, j=1,2,\ldots,k\Big\}.
$$
We bound the entropy of a single block. Let
\begin{eqnarray*}
\mathfrak{D}_{\nu, \ell}^{\rm (block)}=\Big\{\sgn(u_1(a_1))\ldots\sgn(u_k(a_k)){\bf 1}\big\{(a_1,\ldots,a_k)\in A_1\times\ldots\times A_k\big\}:\\
h(A_1\times\ldots A_k)\leq \nu, \prod_{j=1}^k|A_j|=\ell\Big\}.
\end{eqnarray*}
By definition, we obtain
$$
\max\big(|A_1|^2,\ldots, |A_k|^2\big)\leq \nu|A_1||A_2|\ldots |A_k|\leq \nu\ell.
$$
By dividing $\mathfrak{D}_{\nu,\ell}^{(\rm block)}$ into subsets according to $(\ell_1,\ldots,\ell_k)=(|A_1|,\ldots,|A_k|)$, we find
$$
\big|\mathfrak{D}_{\nu,\ell}^{\rm (block)}\big|\leq \sum_{\ell_1\ldots\ell_k=\ell, \max_{j}\ell_j\leq \sqrt{\nu\ell}} 2^{\ell_1+\ldots+\ell_k}{d_1\choose \ell_1}\ldots{d_k\choose \ell_k}.
$$
By the Stirling formula, for $j=1,2,\ldots,k$,
$$
{d_j \choose \ell_j}\leq \frac{d_j^{\ell_j}}{(\ell_j!)}\leq \Big(\frac{d_j}{\ell_j}\Big)^{\ell_j}e^{\ell_j}\frac{1}{\sqrt{2\pi \ell_j}},
$$
then
$$
\log\Big[\sqrt{2\pi\ell_j}2^{\ell_j}{d_j\choose \ell_j}\Big]\leq \ell_j L(\ell_j, 2d_{\max})\leq \sqrt{\nu\ell}L(\sqrt{\nu\ell},2d_{\max})
$$
where $L(x,y):=\max\{1,\log(ey/x)\}$. Let $\ell=\prod_{j=1}^{m}p_j^{v_j}$ with distinct prime factors $p_j$. Since $(v_j+1)v_j/(2p_j^{v_j/2})$ is upper bounded by $2.66$ for $p_j=2$, by $1.16$ for $p_j=3$ and by $1$ for $p_j\geq 5$,
we get
\begin{eqnarray*}
\big|\big\{(\ell_1,\ldots,\ell_k): \ell_1\ldots\ell_k=\ell\big\}\big|=\prod_{j=1}^m{v_j+1\choose k-1}\leq \prod_{j=1}^{m}{v_j+1\choose 2}^{k/2}\\\leq (2.66\times 1.16)^{k/2}(\sqrt{\ell})^{k/2}
\leq \prod_{j=1}^k\big(2\sqrt{2\pi\ell_j}\big)^{k/2},\quad \forall \prod_{j=1}^k\ell_j=\ell.
\end{eqnarray*}
Therefore,
\begin{eqnarray*}
\big|\mathfrak{D}_{\nu,\ell}^{\rm(block)}\big|\leq \frac{\exp\Big(k\sqrt{\nu \ell}L(\sqrt{\nu\ell}, 2d_{\max})\Big)}{\prod_{j=1}^k\sqrt{2\pi \ell_j}}\prod_{j=1}^k\big(2\sqrt{2\pi \ell_j}\big)^{k/2},\quad \forall (\ell_1\ldots\ell_k)=\ell\\
\leq 2^{k^2/2}(2\pi)^{k(k-2)/4}\ell^{(k-2)/4}\exp\Big(k\sqrt{\nu\ell} L(\sqrt{\nu \ell}, 2d_{\max})\Big)\\
\leq 2^{k^2/2}(2\pi)^{k(k-2)/4}\exp\Big(2k\sqrt{\nu\ell}L(\sqrt{\nu\ell}, 2d_{\max})\Big).
\end{eqnarray*}

Due to the constraint $b_1+b_2+\ldots+b_k=s$ in defining $\mathfrak{B}^{\star}_{\nu,m_{\star}}$, for any $\bY\in\mathfrak{B}^{\star}_{\nu,m_{\star}}$, $\bD_{s}(\bY)$ is composed of at most $i^{\star}:={s+k-1\choose k-1}$ blocks. Since the sum of the sizes of the blocks is bounded by $2^{q}$, we obtain
\begin{eqnarray*}
\big|\mathfrak{D}_{\nu,s,q}\big|\leq \sum_{\ell_1+\ldots+\ell_{i^{\star}}\leq 2^{q}} \prod_{i=1}^{i^{\star}}\big|\mathfrak{D}_{\nu,\ell_i}^{\rm(block)}\big|\leq \sum_{\ell_1+\ldots+\ell_{i^{\star}}\leq 2^{q}} (2\pi)^{i^{\star}k(k-2)/4}2^{i^{\star}k^2/2}\exp\Big(2k\sum_{i=1}^{i^{\star}}\sqrt{\nu\ell_i}L(\sqrt{\nu\ell_i}, 2d_{\max})\Big)\\
\leq 2^{i^{\star}k^2/2}(2^q)^{i^{\star}}(2\pi)^{i^{\star}k(k-2)/4}\max_{\ell_1+\ldots+\ell_{i^{\star}}\leq 2^q}\exp\Big(2k\sum_{i=1}^{i^{\star}}\sqrt{\nu\ell_i}L(\sqrt{\nu\ell_i}, 2d_{\max})\Big).
\end{eqnarray*}
As shown in \cite{yuan2015tensor}, $\sum_{i=1}^{i^{\star}}\sqrt{\ell_i}L(\sqrt{\nu\ell_i},2d_{\max})\leq \sqrt{i^{\star}2^q}\big(L(\sqrt{\nu2^q},2d_{\max})+\log(\sqrt{i^{\star}})\big)$, we obtain
\begin{eqnarray*}
\log \big|\mathfrak{D}_{\nu,s,q}\big|\leq i^{\star}\log(2^q)+i^{\star}k(k-2)/2+i^{\star}k^2/2+2k\sqrt{i^{\star}\nu2^q}L\big(\sqrt{\nu 2^q}, 2d_{\max}\sqrt{i^{\star}}\big).
\end{eqnarray*}
Since $i^{\star}={s+k-1\choose k-1}\leq s^k$, it follows that 
\begin{eqnarray*}
\log \big|\mathfrak{D}_{\nu,s,q}\big|\leq qs^k\log2+2k^2s^{k}\sqrt{\nu 2^q} L\big(\sqrt{\nu 2^q}, d_{\max}s^{k/2}\big).
\end{eqnarray*}

\section{Proof of Proposition~\ref{pr:hosvd}}
Recall that $\hat{U}_j^{\rm HOSVD}$ is the top-$r_j$ left singular vectors of $\calM_j\big(\hat{\bT}^{\rm init}\big)$ which can be written as
$$
\calM_j\big(\hat{\bT}^{\rm init}\big)=\calM_j\big(\bT\big)+\calM_j\big(\hat{\bT}^{\rm init}-\bT\big),
$$
and $U_j$ is the top-$r_j$ left singular vectors of $\calM_j(\bT)$. It suffices to study the upper bound of $\big\|\calM_j\big(\hat{\bT}^{\rm init}-\bT\big)\big\|$. Write
\begin{eqnarray}
\calM_j\big(\hat{\bT}^{\rm init}-\bT\big)=\frac{d_1\ldots d_k}{n}\sum_{i=1}^n \xi_i\calM_j(\be_{\omega_i})+\Big(\frac{d_1\ldots d_k}{n}\sum_{i=1}^n\big<\bT,\calM_j(\be_{\omega_i}) \big>\calM_j(\be_{\omega_i})-\bT\Big),
\label{eq:mjhatT-T}
\end{eqnarray}
whose upper bound in operator norm can be derived by matrix Bernstein inequality. For instance, it is easy to check that
$$
\max\Big\{\big\|\mathbb{E}\xi^2\calM_j(\be_\omega)\calM_j^{\top}(\be_{\omega})\big\|, \big\|\EE\xi^2\calM_j^{\top}(\be_{\omega})\calM_j(\be_{\omega})\big\|\Big\}\leq \frac{\sigma_{\xi}^2}{d_1\ldots d_k}\Big(d_j\vee \frac{d_1\ldots d_k}{d_j}\Big),
$$
and
$$
\big\|\|\xi\calM_j(\be_{\omega})\|\big\|_{\psi_2}\leq \|\xi\|_{\psi_2}\lesssim {\sigma_{\xi}}.
$$
By matrix Bernstein inequality (Lemma~\ref{lemma:matBern-bounded} and Lemma~\ref{lemma:matBern-unbound}), the following bound holds with probability at least $1-d_{\max}^\alpha$ for $\alpha\geq 1$
\begin{eqnarray*}
\Big\|\frac{d_1\ldots d_k}{n}\sum_{i=1}^n \xi_i\calM_j(\be_{\omega_i})\Big\|
\leq C_1\sigma_{\xi}\sqrt{\big(d_j\vee (d_1\ldots d_k/d_j)\big)\frac{\alpha kd_1\ldots d_k\log (d_{\max})}{n}}\\
+C_1\sigma_{\xi}\frac{\alpha kd_1\ldots d_k\log(d_{\max})}{n}.
\end{eqnarray*}
Similar bounds can be obtained for the second term in (\ref{eq:mjhatT-T}) and we conclude that with probability at least $1-d_{\max}^{-\alpha}$,
\begin{eqnarray*}
\big\|\calM_j\big(\hat{\bT}^{\rm init}-\bT\big)\big\|\leq C_1\big(\sigma_{\xi}\vee \|\bT\|_{\ell_\infty}\big)\sqrt{\big(d_j\vee (d_1\ldots d_k/d_j)\big)\frac{\alpha kd_1\ldots d_k\log (d_{\max})}{n}}\\
+C_1\big(\sigma_{\xi}\vee \|\bT\|_{\ell_\infty}\big)\frac{\alpha kd_1\ldots d_k\log(d_{\max})}{n}.
\end{eqnarray*}
The claim follows, again, by applying the Wedin's $\sin\Theta$ theorem.

\section{Proof of Proposition~\ref{pr:coherence}}
It is not hard to see that
\begin{eqnarray*}
\|\bA\|_{\ell_\infty}&=&\max_{(i_1,\ldots,i_k)\in [d_1]\times\cdots\times[d_k]}\left|\langle \bA, e_{i_1}\otimes\cdots\otimes e_{i_k}\rangle\right|\\
&=&\max_{(i_1,\ldots,i_k)\in [d_1]\times\cdots\times[d_k]}\left|\langle \bA, (U_1U_1^\top e_{i_1})\otimes\cdots\otimes (U_kU_k^\top e_{i_k})\rangle\right|\\
&\le&\|\bA\|_{\ell_2}\left(\max_{i_1\in[d_1]}\|U_1^\top e_{i_1}\|_{\ell_2}\right)\cdots \left(\max_{i_k\in[d_k]}\|U_k^\top e_{i_k}\|_{\ell_2}\right)\\
&\le&\|\bA\|_{\ell_2}\mu^{k/2}(\bA)\sqrt{r_1(\bA)\cdots r_k(\bA)\over d_1\cdots d_k},
\end{eqnarray*}
so that
$$
\beta(\bA)\le r_1^{1/2}(\bA)\cdots r_k^{1/2}(\bA)\mu^{k/2}(\bA).
$$
Here $e_{i_j}$ is the $i_j$-th canonical basis of an Euclidean space $\RR^{d_j}$.

On the other hand, we show $\mu(U_1)\leq \kappa^2(\bA)\beta^2(\bA)$. Denote by $A_1=\calM_1(\bA)$ and $C_1=\calM_1(\bC)$ where $\bC\in\RR^{r_1(\bA)\times\ldots\times r_k(\bA)}$ denotes the core tensor of $\bA$. Then
$$
A_1=U_1C_1\Big(\bigotimes_{j>1}U_j\Big)^{\top}\in\mathbb{R}^{d_1\times (d_2\ldots d_k)}.
$$
For any integer $1\leq i\leq d_1$, we denote by $A_1(i,\cdot)$ and $U_1(i,\cdot)$ the $i$-th row of $A_1$ and $U_1$ respectively. Then,
$$
(d_2\ldots d_k)\|\bA\|_{\ell_\infty}^2\geq \|A_1(i,\cdot)\|_{\ell_2}^2=U_1(i,\cdot)C_1C_1^{\top}U_1(i,\cdot)^{\top}\geq \sigma_{\min}^2(A_1)\|P_{U_1}e_i\|_{\ell_2}^2.
$$
As a result,
$$
\|P_{U_1}e_i\|_{\ell_2}^2\leq \frac{d_2\ldots d_k\|\bA\|_{\ell_\infty}^2}{\sigma_{\min}^2(A_1)}\leq \frac{\beta^2(\bA)}{d_1}\frac{\|\bA\|_{\ell_2}^2}{\sigma_{\min}^2(A_1)}\leq \kappa^2(\bA)\beta^2(\bA)\frac{r_1(\bA)}{d_1}
$$
implying that
$$
\|P_{U_1}e_i\|_{\ell_2}\leq \kappa(\bA)\beta(\bA)\sqrt{\frac{r_1(\bA)}{d_1}},\quad 1\leq i\leq d_1,
$$
which concludes the proof.

\section{Proof of Lemma~\ref{lemma:hatR}}
Without loss of generality, consider $j=1$ and our goal is to prove the upper bound of 
$$
\big\|\calM_1\big((\hat{\bT}^{\rm init}-\bT)\times_{j'>1}(P_{U_{j'}}U_{j'}^{\rm(iter)})\big)\big\|,
$$
and $\|\hat{\bR}\|$, where $\hat{\bR}$ can be explicitly expressed as
\begin{equation}\label{eq:hatRineq1}
\hat{\bR}=\sum_{s=1, s\neq j}^{k}\calM_1\Big((\hat{\bT}^{\rm init}-\bT)\times_{j'<s}(P_{U_{j'}}U_{j'}^{\rm(iter)})\times_s(P_{U_{s}}^{\perp}U_{s}^{\rm(iter)})\times_{j'>s}U_{j'}^{\rm(iter)}\Big).
\end{equation}
\paragraph{Proof of first claim} Observe that $P_{U_{j'}}U_{j'}^{\rm(iter)}=U_{j'}\big(U_{j'}^{\top}U_{j'}^{\rm(iter)}\big)$ and $\|U_{j'}^{\top}U_{j'}^{\rm(iter)}\|\leq 1$. Therefore, it suffices to prove the upper bound of 
$$
\Big\|\calM_1\big((\hat{\bT}^{\rm init}-\bT)\times_{j'>1} U_{j'}\big)\Big\|.
$$
We write 
\begin{eqnarray*}
\calM_1\big((\hat{\bT}^{\rm init}-\bT)\times_{j'>1} U_{j'}\big)=\calM_1(\hat{\bT}^{\rm init}-\bT)\Big(\bigotimes_{j'>1}U_{j'}\Big)=:\bDelta+{\bf \Xi}.
\end{eqnarray*}
Here
$$
\bDelta:=\frac{d_1\ldots d_k}{n}\sum_{i=1}^n\langle\bT,\be_{\omega_i} \rangle\calM_1(\be_{\omega_i})\Big(\bigotimes_{j'>1}U_{j'}\Big)-\calM_1(\bT)\Big(\bigotimes_{j'>1}U_{j'}\Big),
$$
and 
$$
{\bf \Xi}:=\frac{d_1\ldots d_k}{n}\sum_{i=1}^n\xi_i\calM_1(\be_{\omega_i})\Big(\bigotimes_{j'>1}U_{j'}\Big),
$$
where each term can be bounded by matrix Bernstein inequalities.

 For notational simplicity, we write $\bX_i$ in short of $\calM_1(\be_{\omega_i})\in\mathbb{R}^{d_1\times (d_2\ldots d_k)}$ and $\bM=\calM_1(\bT)\in\RR^{d_1\times (d_2\ldots d_k)}$. It is easy to check the following bounds
$$
\Big\|\mathbb{E}\xi_i^2\big(\bigotimes_{j'>1}U_{j'}\big)^{\top}\bX_i^{\top}\bX_i\big(\bigotimes_{j'>1}U_{j'}\big)\Big\|\leq \frac{d_1\sigma_{\xi}^2}{d_1\ldots d_k},
$$
and
$$
\big\|\mathbb{E}\xi_i^2\bX_i\big(\bigotimes_{j'>1}U_{j'}\big)\big(\bigotimes_{j'>1}U_{j'}\big)^{\top}\bX_i^{\top}\big\|\leq \frac{\sigma_{\xi}^2}{d_1\ldots d_k}\tr\Big(\big(\bigotimes_{j'>1}U_{j'}\big)\big(\bigotimes_{j'>1}U_{j'}\big)^{\top}\Big)\leq \frac{r_1\ldots r_k\sigma_{\xi}^2}{r_1d_1\ldots d_k}.
$$
Moreover,
$$
\Big\|\big\|\xi\bX\big(\bigotimes_{j'>1}U_{j'}\big)\big\|\Big\|_{\psi_2}\leq \|\xi\|_{\psi_2}\big\|\bX\big(\bigotimes_{j'>1}U_{j'}\big)\big\|\lesssim \sigma_{\xi}\big(\beta_0\kappa(\bT)\big)^{k-1}\Big(\frac{r_1\ldots r_k}{d_1\ldots d_k}\Big)^{1/2}\Big(\frac{d_1}{r_1}\Big)^{1/2},
$$
where we used the fact that $\bT\in\Theta(r_1,\ldots,r_k)$ and Proposition~\ref{pr:coherence} such that
$$
\big\|\big(U_{j'}\big)_{i\cdot}\big\|_{\ell_2}\leq \beta(\bT)\kappa(\bT)\sqrt{\frac{r_{j'}}{d_{j'}}}\quad \forall\ 1\leq j'\leq k, 1\leq i\leq d_{j'}.
$$
By matrix Bernstein inequality (Lemma~\ref{lemma:matBern-bounded} and Lemma~\ref{lemma:matBern-unbound}), the following bound holds with probability at least $1-d_{\max}^{-\alpha}$,
\begin{align*}
\|{\bf \Xi}\|\leq C_1\sigma_{\xi}\max\bigg\{(d_{1}\vee r_2\ldots r_k)^{1/2}&\frac{(\alpha kd_1\ldots d_k\log (d_{\max}))^{1/2}}{n^{1/2}},\\
&\big(\beta(\bT)\kappa(\bT)\big)^{k-1}\frac{\alpha k(r_1\ldots r_kd_1\ldots d_k)^{1/2}\log (d_{\max})}{n}\Big(\frac{d_1}{r_1}\Big)^{1/2}\bigg\}.
\end{align*}
where the first term dominates if $n\geq C_2k\alpha \big(\beta(\bT)\kappa(\bT)\big)^{2(k-1)}d_1\log (d_{\max})$ for some absolute constants $C_1,C_2>0$.

In a similar fashion, we can obtain the upper bound of $\|\bDelta\|$. Therefore, we conclude that if $n\geq C_1k\alpha\big(\beta(\bT)\kappa(\bT)\big)^{2(k-1)}d_1\log (d_{\max})$, then with probability at least $1-d_{\max}^{-\alpha}$ that
$$
\Big\|\calM_1\big((\hat{\bT}^{\rm init}-\bT)\times_{j'>1}U_{j'}\big)\Big\|\leq C_2\big(\|\bT\|_{\ell_\infty}\vee \sigma_{\xi}\big)(d_1\vee r_2\ldots r_k)^{1/2}\sqrt{\frac{\alpha kd_1\ldots d_{k}\log (d_{\max})}{n}},
$$
for some absolute constants $C_1,C_2>0$. 

\paragraph{Proof of the second claim} Recall the representation of $\hat{\bR}$ as (\ref{eq:hatRineq1}), it suffices to prove the upper bound, for each $s\in[k]$ and $s\neq j$,  of 
$$
\Big\|\calM_1\Big((\hat{\bT}^{\rm init}-\bT)\times_{j'<s}\big(P_{U_{j'}}U_{j'}^{(\rm iter)}\big)\times_s\big(P_{U_s}^{\perp}U_s^{\rm (iter)}\big)\times_{j'>s}U_{j'}^{\rm(iter)}\Big)\Big\|,
$$
which is a matrix of size $d_1\times (r_2r_3\ldots r_k)$. To this end, we need the following simple fact.
\begin{lemma}\label{lemma:tensornorm}
For a tensor $\bA\in\RR^{d_1\times\ldots d_k}$ with multilinear ranks $(r_1,\ldots,r_k)$, the following bound holds for all $j=1,2,\ldots,k$,
\begin{eqnarray*}
\big\|\calM_j(\bA)\big\|\leq \|\bA\|\sqrt{\frac{(r_1\ldots r_k)/r_j}{\max_{j'\neq j}r_{j'}}}.
\end{eqnarray*}
\end{lemma}
\begin{proof}[Proof of Lemma~\ref{lemma:tensornorm}]
Let $\bC\in\RR^{r_1\times\ldots\times r_k}$ be the core tensor of $\bA$. Without loss of generality, consider $j=1$. It suffices to show that
$$
\big\|\calM_1(\bC)\big\|\leq \|\bC\|\sqrt{\frac{r_2r_3\ldots r_k}{\max_{2\leq j\leq k}r_j}}.
$$
Recall that $\calM_1(\bC)\in\RR^{r_1\times (r_2\ldots r_k)}$. 
 Let $u\in \mathbb{R}^{r_1}, v\in\mathbb{R}^{r_2\ldots r_k}$ with $\max\big\{\|u\|_{\ell_2},\|v\|_{\ell_2}\big\|\leq 1$. By definition,
\begin{eqnarray*}
\|\calM_1(\bC)\|:=\underset{u,v}{\sup}\ \big<\calM_1(\bC),u\otimes v \big>.
\end{eqnarray*}
Now let $\bC_{2,s}, s=1,2,\ldots, (r_3\ldots r_k) $ denote the mode-$2$ slices of $\bC$. In other words, all the matrices $\bC_{2,s}$ has size $r_1\times r_2$ and
$$
\bC_{2, s}=\Big(C(i_1,i_2,j_3,j_4,\ldots,j_k)\Big)_{i_1\in[r_1],i_2\in[r_2]}
$$
with 
$$
s=\Big(\sum_{q=3}^k (j_q-1)\big(\prod_{p=q+1}^kr_p\big)\Big)+1.
$$
Similarly, let $\{v_1,\ldots,v_{r_3\ldots r_k}\}\subset \mathbb{R}^{r_2}$ denote the corresponding segments of $v$ each has size $r_2$. Then
\begin{eqnarray*}
\|\calM_1(\bC)\|\leq \underset{u,v}{\sup} \sum_{s=1}^{r_3\ldots r_k}\|\bC_{2,s}\|\|v_{s}\|\leq \|\bC\| \underset{u,v}{\sup}\sum_{s=1}^{r_3\ldots r_k}\|v_s\|\leq \|\bC\|\sqrt{r_3\ldots r_k}\\
=\|\bC\|\sqrt{\frac{r_2r_3\ldots r_k}{r_2}}
\end{eqnarray*}
where the last inequality is due to Cauchy-Schwarz inequality. Similar bounds can be attained through other matricizations of $\bC$ and we end up with the bound
$$
\big\|\calM_1(\bC)\big\|\leq \|\bC\|\sqrt{\frac{r_2r_3\ldots r_k}{\max_{2\leq j\leq k}r_j}}.
$$
For general $j=1,2,\ldots,k$, the conclusion becomes
$$
\big\|\calM_j\big(\bA\big)\big\|\leq \|\bA\|\sqrt{\frac{(r_1\ldots r_k)/r_j}{\max_{j'\neq j}r_{j'}}}.
$$
\end{proof}

Based on Lemma~\ref{lemma:tensornorm}, we conclude that
\begin{eqnarray*}
\Big\|\calM_1\Big((\hat{\bT}^{\rm init}-\bT)\times_{j'<s}\big(P_{U_{j'}}U_{j'}^{(\rm iter)}\big)\times_s\big(P_{U_s}^{\perp}U_s^{\rm (iter)}\big)\times_{j'>s}U_{j'}^{\rm(iter)}\Big)\Big\|\\
\leq\big\|\big(U_s^{\perp}\big)^{\top}U_s^{\rm(iter)}\big\|\sqrt{\frac{(r_1\ldots r_k)/r_1}{\max_{j'\neq 1}r_{j'}}}\|\hat{\bT}^{\rm init}-\bT\|\leq E_{\rm iter}\sqrt{\frac{(r_1\ldots r_k)/r_1}{\max_{j'\neq 1}r_{j'}}}\|\hat{\bT}^{\rm init}-\bT\|,
\end{eqnarray*}
where we used the fact
$$
\big\|\big(U_{s}^{\perp}\big)^{\top}U_s^{\rm(iter)}\big\|=\big\|\big(U_{s}^{\perp}\big)^{\top}\big(U_sU_s^{\top}-U_s^{\rm(iter)}(U_s^{\rm(iter)})^{\top}\big)U_s^{\rm(iter)}\big\|\leq E_{\rm iter}.
$$
By Theorem~\ref{th:init}, we know that the following bound holds with probability at least $1-d_{\max}^{-\alpha}$,
\begin{eqnarray*}
\Big\|\calM_1\Big((\hat{\bT}^{\rm init}-\bT)\times_{j'<s}\big(P_{U_{j'}}U_{j'}^{(\rm iter)}\big)\times_s\big(P_{U_s}^{\perp}U_s^{\rm (iter)}\big)\times_{j'>s}U_{j'}^{\rm(iter)}\Big)\Big\|\\
\leq E_{\rm iter} C_1\alpha k^{k+3}\sqrt{\frac{(r_1\ldots r_k)/r_1}{\max_{j'\neq 1}r_{j'}}}\big(\|\bT\|_{\ell_\infty}\vee \sigma_{\xi}\big)\max\bigg\{\sqrt{\frac{kd_{\max}d_1\ldots d_k}{n}},\quad \frac{kd_1\ldots d_k}{n}\bigg\}\log^{k+2}d_{\max},
\end{eqnarray*}
which obviously concludes the proof by observing that $\max_{1\leq j\leq k}r_j\leq r_{\max}(\bT)$.

\section{Proof of Lemma~\ref{lemma:TUiter-U}}
We need to obtain the upper bound 
\begin{align*}
&\Big\|(P_{U_{j}^{(\rm iter)}}-P_{U_j})\calM_j(\bT)\big(\bigotimes_{j'<j}P_{U_{j'}^{(\rm iter)}}\bigotimes_{j'>j}P_{U_{j'}}\big)\Big\|_{\ell_2}\\
=&
\Big\|(P_{U_{j}^{(\rm iter)}}-P_{U_j})\calM_j(\bT)\big(\bigotimes_{j'<j}P_{U_{j'}}P_{U_{j'}^{(\rm iter)}}\bigotimes_{j'>j}P_{U_{j'}}\big)\Big\|_{\ell_2}.
\end{align*}
Observe that $(I-P_{U_j})\calM_j(\bT)=0$, we write
\begin{align*}
&\Big\|(P_{U_{j}^{(\rm iter)}}-P_{U_j})\calM_j(\bT)\big(\bigotimes_{j'<j}P_{U_{j'}}P_{U_{j'}^{(\rm iter)}}\bigotimes_{j'>j}P_{U_{j'}}\big)\Big\|\\
\leq& \Big\|(P_{U_{j}^{(\rm iter)}}-I)\calM_j(\hat\bT^{\rm init})\big(\bigotimes_{j'<j}P_{U_{j'}}P_{U_{j'}^{(\rm iter)}}\bigotimes_{j'>j}P_{U_{j'}}\big)\Big\|\\
&\hskip 100pt+\Big\|(P_{U_{j}^{(\rm iter)}}-I)\calM_j(\hat\bT^{\rm init}-\bT)\big(\bigotimes_{j'<j}P_{U_{j'}}P_{U_{j'}^{(\rm iter)}}\bigotimes_{j'>j}P_{U_{j'}}\big)\Big\|\\
\leq& \Big\|(P_{U_{j}^{(\rm iter)}}-I)\calM_j(\hat\bT^{\rm init})\big(\bigotimes_{j'<j}P_{U_{j'}}P_{U_{j'}^{(\rm iter)}}\bigotimes_{j'>j}P_{U_{j'}}\big)\Big\|\\
&\hskip 100pt+C_2\big(\|\bT\|_{\ell_\infty}\vee \sigma_{\xi}\big)\Big(d_j\vee \frac{r_1\ldots r_k}{r_j}\Big)^{1/2}\sqrt{\frac{\alpha k d_1\ldots d_k\log d_{\max}}{n}}
\end{align*}
where the last inequality is due to Lemma~\ref{lemma:hatR}.  Moreover, we write
\begin{align*}
&\Big\|(P_{U_{j}^{(\rm iter)}}-I)\calM_j(\hat\bT^{\rm init})\big(\bigotimes_{j'<j}P_{U_{j'}}P_{U_{j'}^{(\rm iter)}}\bigotimes_{j'>j}P_{U_{j'}}\big)\Big\|\\
\leq& \Big\|(P_{U_{j}^{(\rm iter)}}-I)\calM_j(\hat\bT^{\rm init})\bigotimes_{j'\neq j}P_{U_{j'}^{(\rm iter)}}\Big\|\\
&\hskip 50pt+\Big\|(P_{U_{j}^{(\rm iter)}}-I)\calM_j(\hat\bT^{\rm init})\Big(\bigotimes_{j'\neq j}P_{U_{j'}^{(\rm iter)}}-\bigotimes_{j'<j}P_{U_{j'}}P_{U_{j'}^{(\rm iter)}}\bigotimes_{j'>j}P_{U_{j'}}\Big)\Big\|.
\end{align*}
Recall that $U_j^{(\rm iter)}$ are the top $r_j$ left singular vectors of $\calM_j(\hat\bT^{\rm init})\bigotimes_{j'\neq j}P_{U_{j'}^{(\rm iter)}}$. Therefore, 
\begin{align*}
\Big\|(P_{U_{j}^{(\rm iter)}}-I)\calM_j(\hat\bT^{\rm init})\bigotimes_{j'\neq j}P_{U_{j'}^{(\rm iter)}}\Big\|\leq& \sigma_{r_j+1}\Big(\calM_j(\hat\bT^{\rm init})\bigotimes_{j'\neq j}P_{U_{j'}^{(\rm iter)}}\Big)\\
\leq& \Big\|\calM_j(\hat\bT^{\rm init}-\bT)\bigotimes_{j'\neq j}P_{U_{j'}^{(\rm iter)}}\Big\|\\
\leq& C_1\big(\|\bT\|_{\ell_\infty}\vee \sigma_{\xi}\big)\Big(d_j\vee \frac{r_1\ldots r_k}{r_j}\Big)^{1/2}\sqrt{\frac{\alpha k d_1\ldots d_k\log d_{\max}}{n}}+\\
&+C_2 E_{{\rm iter}_{\max}}\alpha k^{k+4}r_{\max}(\bT)^{(k-2)/2}\big(\|\bT\|_{\ell_\infty}\vee \sigma_{\xi}\big)\times\\
&\times\max\Big\{\sqrt{\frac{kd_{\max}d_1\ldots d_n}{n}}, \frac{kd_1\ldots d_k}{n}\Big\}\log^{k+2}d_{\max},
\end{align*}
where the last inequality is again due to Lemma~\ref{lemma:hatR}. We then apply Lemma~\ref{lemma:hatR} to the following term,
\begin{align*}
&\Big\|(P_{U_{j}^{(\rm iter)}}-I)\calM_j(\hat\bT^{\rm init})\Big(\bigotimes_{j'\neq j}P_{U_{j'}^{(\rm iter)}}-\bigotimes_{j'<j}P_{U_{j'}}P_{U_{j'}^{(\rm iter)}}\bigotimes_{j'>j}P_{U_{j'}}\Big)\Big\|\\
\leq& \Big\|(P_{U_{j}^{(\rm iter)}}-I)\calM_j(\hat\bT^{\rm init}-\bT)\Big(\bigotimes_{j'\neq j}P_{U_{j'}^{(\rm iter)}}-\bigotimes_{j'<j}P_{U_{j'}}P_{U_{j'}^{(\rm iter)}}\bigotimes_{j'>j}P_{U_{j'}}\Big)\Big\|\\
&\hskip 50pt+\Big\|(P_{U_{j}^{(\rm iter)}}-I)\calM_j(\bT)\Big(\bigotimes_{j'\neq j}P_{U_{j'}^{(\rm iter)}}-\bigotimes_{j'<j}P_{U_{j'}}P_{U_{j'}^{(\rm iter)}}\bigotimes_{j'>j}P_{U_{j'}}\Big)\Big\|\\
\leq& C_2 E_{{\rm iter}_{\max}}\alpha k^{k+4}r_{\max}(\bT)^{(k-2)/2}\big(\|\bT\|_{\ell_\infty}\vee \sigma_{\xi}\big)\max\Big\{\sqrt{\frac{kd_{\max}d_1\ldots d_n}{n}}, \frac{kd_1\ldots d_k}{n}\Big\}\log^{k+2}d_{\max}\\
&\hskip 50pt+\Big\|(P_{U_{j}^{(\rm iter)}}-I)\calM_j(\bT)\Big(\bigotimes_{j'\neq j}P_{U_{j'}^{(\rm iter)}}-\bigotimes_{j'<j}P_{U_{j'}}P_{U_{j'}^{(\rm iter)}}\bigotimes_{j'>j}P_{U_{j'}}\Big)\Big\|.
\end{align*}
Again by the fact $(I-P_{U_j})\calM_j(\bT)=0$, we have 
\begin{align*}
&\Big\|(P_{U_{j}^{(\rm iter)}}-I)\calM_j(\bT)\Big(\bigotimes_{j'\neq j}P_{U_{j'}^{(\rm iter)}}-\bigotimes_{j'<j}P_{U_{j'}}P_{U_{j'}^{(\rm iter)}}\bigotimes_{j'>j}P_{U_{j'}}\Big)\Big\|\\
=&\Big\|(P_{U_{j}^{(\rm iter)}}-P_{U_j})\calM_j(\bT)\Big(\bigotimes_{j'\neq j}P_{U_{j'}^{(\rm iter)}}-\bigotimes_{j'<j}P_{U_{j'}}P_{U_{j'}^{(\rm iter)}}\bigotimes_{j'>j}P_{U_{j'}}\Big)\Big\|\\
\leq& k\Lambda_{\max}(\bT)E_{{\rm iter}_{\max}}^2.
\end{align*}
To sum up, we obtain
\begin{align*}
&\Big\|(P_{U_{j}^{(\rm iter)}}-P_{U_j})\calM_j(\bT)\big(\bigotimes_{j'<j}P_{U_{j'}}P_{U_{j'}^{(\rm iter)}}\bigotimes_{j'>j}P_{U_{j'}}\big)\Big\|\\
\leq& C_2\big(\|\bT\|_{\ell_\infty}\vee \sigma_{\xi}\big)\Big(d_j\vee \frac{r_1\ldots r_k}{r_j}\Big)^{1/2}\sqrt{\frac{\alpha k d_1\ldots d_k\log d_{\max}}{n}}+k\Lambda_{\max}(\bT)E_{{\rm iter}_{\max}}^2+\\
&\hskip 50pt+C_2 E_{{\rm iter}_{\max}}\alpha k^{k+4}r_{\max}(\bT)^{(k-2)/2}\big(\|\bT\|_{\ell_\infty}\vee \sigma_{\xi}\big)\times\\
&\hskip 150pt\times\max\Big\{\sqrt{\frac{kd_{\max}d_1\ldots d_n}{n}}, \frac{kd_1\ldots d_k}{n}\Big\}\log^{k+2}d_{\max}.
\end{align*}
Since 
$$
E_{{\rm iter}_{\max}}\leq C_2 2^{k/2}\frac{\big(\|\bT\|_{\ell_\infty}\vee \sigma_{\xi}\big)}{\Lambda_{\min}(\bT)}\Big(d_{\max}\vee \frac{r_1\ldots r_k}{\min_{1\leq j\leq k} r_j}\Big)^{1/2}\sqrt{\frac{\alpha k d_1\ldots d_k\log d_{\max}}{n}},
$$
we conclude that, with probability at least $1-d_{\max}^{-\alpha}$,
\begin{align*}
&\Big\|(P_{U_{j}^{(\rm iter)}}-P_{U_j})\calM_j(\bT)\big(\bigotimes_{j'<j}P_{U_{j'}}P_{U_{j'}^{(\rm iter)}}\bigotimes_{j'>j}P_{U_{j'}}\big)\Big\|\\
\leq& C_2\big(\|\bT\|_{\ell_\infty}\vee \sigma_{\xi}\big)\Big(d_{\max}\vee \frac{r_1\ldots r_k}{\min_{1\leq j\leq k} r_j}\Big)^{1/2}\sqrt{\frac{\alpha k d_1\ldots d_k\log d_{\max}}{n}},
\end{align*}
as long as $n$ satisfies the requirement given in Lemma~\ref{lemma:TUiter-U}. Since 
$$
(P_{U_{j}^{(\rm iter)}}-P_{U_j})\calM_j(\bT)\big(\bigotimes_{j'<j}P_{U_{j'}}P_{U_{j'}^{(\rm iter)}}\bigotimes_{j'>j}P_{U_{j'}}\big)
$$
 has rank at most $2r_j$, we obtain
\begin{align*}
\Big\|(P_{U_{j}^{(\rm iter)}}-P_{U_j})\calM_j(\bT)\big(\bigotimes_{j'<j}P_{U_{j'}}P_{U_{j'}^{(\rm iter)}}\bigotimes_{j'>j}P_{U_{j'}}\big)\Big\|_{\ell_2}\\
\leq C_2\big(\|\bT\|_{\ell_\infty}\vee \sigma_{\xi}\big)\Big(d_{\max}r_{\max}(\bT)\vee \frac{r_1\ldots r_k}{\min_{1\leq j\leq k} r_j/r_{\max}(\bT)}\Big)^{1/2}\sqrt{\frac{\alpha k d_1\ldots d_k\log d_{\max}}{n}}.
\end{align*}
\end{document}